%% file: manuscript_report.tex
\documentclass[letterpaper, 10 pt, conference]{ieeeconf}  

\IEEEoverridecommandlockouts                              

\input{packages}
\input{commands}

\input{notation}

\title{\LARGE \bf
\mbox{Graph-Theoretic B\'ezier Curve Optimization over Safe Corridors} \\ for Safe and Smooth Motion Planning \\[2mm] \large{(Technical Report)}
}

\author{Soufyan Zayou and \"{O}m\"{u}r Arslan
\thanks{The authors are with the Department of Mechanical Engineering, Eindhoven University of Technology, P.O. Box 513, 5600 MB Eindhoven, The Netherlands. The second author is also affiliated with the Eindhoven AI Systems Institute. Emails: s.zayou@student.tue.nl,  o.arslan@tue.nl}%
}

\begin{document}

\maketitle
\thispagestyle{empty}
\pagestyle{empty}

\begin{abstract}
As a parametric motion representation, B\'ezier curves have significant applications in polynomial trajectory optimization for safe and smooth motion planning of various robotic systems, including flying drones, autonomous vehicles, and robotic manipulators.
An essential component of B\'ezier curve optimization is the optimization objective, as it significantly influences the resulting robot motion.
Standard physical optimization objectives, such as minimizing total velocity, acceleration, jerk, and snap, are known to yield quadratic optimization of B\'ezier curve control points.
In this paper, we present a unifying graph-theoretic perspective for defining and understanding B\'ezier curve optimization objectives using a \emph{consensus distance} of B\'ezier control points derived based on their interaction graph Laplacian. 
In addition to demonstrating how standard physical optimization objectives define a consensus distance between B\'ezier control points, we also introduce geometric and statistical optimization objectives as alternative consensus distances, constructed using finite differencing and differential variance. 
To compare these optimization objectives,  we apply B\'ezier curve optimization over convex polygonal safe corridors that are automatically constructed around a maximal-clearance minimal-length reference path. 
We provide an explicit analytical formulation for quadratic optimization of B\'ezier curves using B\'ezier matrix operations.
We conclude that the norm and variance of the finite differences of B\'ezier control points lead to simpler and more intuitive interaction graphs and optimization objectives compared to B\'ezier derivative norms, despite having similar robot motion profiles.
\end{abstract}

\section{Introduction}
\label{sec.Introduction}

Safe and smooth motion planning is an essential skill for autonomous robots not only for their own motion safety and control comfort but also for fostering safe, smooth, and predictable interaction with both humans and other robots.
As a parametric polynomial motion representation, B\'ezier curves find significant applications in safe and smooth robot motion planning of various autonomous robots from flying drones \cite{mellinger_kumar_ICRA2011, richter_bry_roy_ISRR016, ding_gao_wang_shen_TRO2019, gao_wu_lin_shen_ICRA2018, tordesillas_etal_TRO2021}  to autonomous vehicles \cite{gonzalez_etal_TITS2016, ding_zhang_chen_shen_RAL2019, qian_etal_ITSC2016, perez_godoy_villagra_onieva_ICRA2013} to robotic manipulators \cite{ozaki_lin_ICRA1996, hauser_ng-thow-hing_ICRA2010,scheiderer_thun_meisen_FAIM2019, zhao_etal_CYBER2019}.
B\'ezier curves are primarily popular because of their key characteristic properties, including interpolation, convexity, derivative, and norm properties, which are essential for computationally efficient quadratic optimization \cite{mellinger_kumar_ICRA2011, bry_richter_bachrach_roy_IJRR2015, gao_wu_lin_shen_ICRA2018, gao_etal_JFR2019}.
The interpolation and derivative properties of B\'ezier curves allow for explicitly defining the endpoint continuity constraints as a linear equality constraint.
The convexity property of B\'ezier curves allows one to express safety and system requirements as a linear inequality constraint.
The norm of high-order derivatives of polynomial B\'ezier curves are known to yield quadratic optimization objectives. 
However, a general, explicit analytical form for the Hessian of B\'ezier derivative norms is lacking in the existing literature, which limits our intuitive understanding of these optimization objectives.

In this paper, we present a new unifying graph-theoretic approach for defining and understanding quadratic B\'ezier optimization objectives through a \emph{consensus distance} of B\'ezier control points, derived from the weighted interaction graph of these control points.
We demonstrate that standard physical B\'ezier optimization objectives, derived based on the derivative norm of B\'ezier curves, define a consensus distance for B\'ezier control points.
Inspired by this observation, as an alternative consensus distance, we construct new geometric and statistical quadratic B\'ezier optimization objectives based on the finite-difference norm and differential variance of B\'ezier control points.
We present a general explicit analytical formula for the Hessian (i.e., graph Laplacian) of these physical, geometric, and statistical optimization objectives.
We apply and compare these quadratic B\'ezier optimization objectives for safe and smooth motion planning over convex polygonal safe corridors that are automatically constructed along a maximal-clearance minimal-length reference path.

\begin{figure}[t]
\begin{tabular}{@{}c@{\hspace{2mm}}c@{}}
\begin{tabular}{@{}c@{}}
\scalebox{0.55}{\input{drawings/tikz_first_order_discrete_difference}} 
\\
\scalebox{0.55}{\input{drawings/tikz_first_order_difference_variance}}
\\
\scalebox{0.55}{\input{drawings/tikz_second_order_discrete_difference_normAlt1}} 
\end{tabular}
&
\begin{tabular}{@{}c@{}}
\includegraphics[width = 0.5\columnwidth]{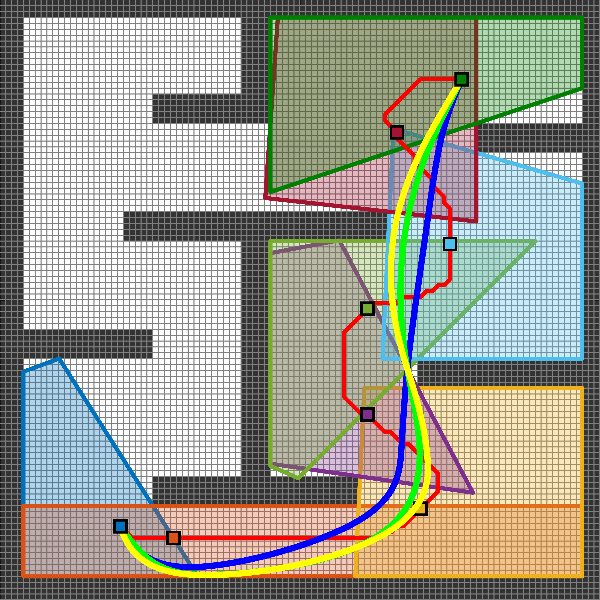} 
\end{tabular}
\end{tabular}
\vspace{-1mm}
\caption{Graph-theoretic B\'ezier curve optimization aims at minimizing the consensus distance $\trace\plist{\tr{\pmat}\mat{L} \pmat}$ of control points $\pmat = \tr{[\ppoint_0, \ldots, \ppoint_n]}$ of B\'ezier curves $\bcurve_{\pmat}(t)$ based on the positive semidefinite Laplacian matrix $\mat{L}$ of an interaction graph of the control points while satisfying linear equality (e.g., endpoint continuity) and inequality (e.g., safety corridor) constraints. The blue, green, and yellow curves demonstrate a collection of B\'ezier curves optimized over safe corridors (polygon patches) with the $C^1$ endpoint continuity constraint using the graph Laplacians of the (left, top) first-order difference norm, (left, middle) first-order difference variance, and (left, bottom) second-order difference norm of B\'ezier control points, respectively.
Convex polygonal safe corridors are automatically generated using separating hyperplanes of corridor centers from obstacles around a maximal-clearance minimal-length reference path (red).} 
\label{fig.GraphTheoreticBezierCurveOptimization}
\vspace{-3mm}
\end{figure}

\subsection{Motivation and Related Literature}

The quadratic nature of the squared norm of high-order derivatives of polynomial B\'ezier curves is a primary reason for their popularity for computationally efficient quadratic optimization in smooth motion planning and path smoothing, especially when dealing with differential flat systems \cite{nieuwstadt_murray_IJRNC1998} such as cars \cite{ding_zhang_chen_shen_RAL2019, qian_etal_ITSC2016}, quadrotors \cite{mellinger_kumar_ICRA2011, richter_bry_roy_ISRR016, ding_gao_wang_shen_TRO2019}, and fixed-wing aircraft \cite{bry_richter_bachrach_roy_IJRR2015}, to name a few, whose control inputs can be expressed as a function of \emph{flat system outputs} (represented by polynomials) and their derivatives.
Most existing work on polynomial trajectory optimization mainly focuses on time-parametrized polynomials to minimize physical quadratic objective functions such as total velocity, acceleration, jerk, and snap \cite{chen_su_shen_ROBIO2015, liu_etal_RAL2017, chen_liu_shen_ICRA2016, gao_etal_JFR2019}.
Inspired by spring forces and potentials, the norm of the second-order finite difference of B\'ezier control points, also known as the elastic band cost \cite{quinlan_khatib_ICRA1993, zhu_schmerling_pavone_CDC2015}, is also used as a geometric optimization objective to decouple time allocation with geometric curve optimization \cite{boyu_etal_RAL2019}.
In all of this existing literature, although the Hessian of a quadratic optimization objective is crucial for quadratic optimization of polynomial curves, it is omitted for brevity, or, in some cases, presented through manual calculation or symbolic and numeric computation.
The lack of a general explicit analytical expression for the Hessian of the optimization function obscures the relationships between potentially related optimization objectives. 
Consequently, it becomes challenging to intuitively understand how these objectives influence polynomial curve parameters.
In this paper, we present an explicit analytical formula for the norm and variance of the high-order derivatives of B\'ezier curves, as well as the finite differences of B\'ezier control points using matrix representations and operations associated with B\'ezier curves.
We also demonstrate that these quadratic objective functions can be unified as a consensus distance of B\'ezier control points, which is derived from the Laplacian matrix of a weighted interaction graph of the control points.
The explicit analytical forms of the Hessian, along with their graph-theoretical interpretations, offer new insights into B\'ezier optimization objectives.
For example, for B\'ezier curves of degree $n = 2$,  B\'ezier velocity norm is equivalent to the variance of the control points, and similarly, B\'ezier acceleration norm is equivalent to the variance of the first-order difference of the control points for $n = 2, 3$.  
We believe that our explicit matrix formulation of B\'ezier curve optimization can reduce the knowledge gap and make it more accessible for researchers to apply these tools effectively.

The convex nature of B\'ezier curves makes them an ideal choice for addressing system and task constraints as a linear inequality constraint in quadratic optimization, particularly for safe and smooth motion planning around obstacles.
Since B\'ezier curves are contained in the convex hull of their control points, such optimization constraints are often enforced by constraining B\'ezier control points inside some convex constraint sets, known as safe corridors \cite{gao_wu_lin_shen_ICRA2018, honig_preiss_kumar_sukhatme_ayanian_TRO2018, choi_curry_elkaim_JAM2010}. 
Safe corridors are often iteratively constructed around a reference path in various shapes, including convex polytopes \cite{liu_etal_RAL2017, deits_tedrake_ICRA2015}, boxes \cite{chen_su_shen_ROBIO2015, chen_liu_shen_ICRA2016, gao_wu_lin_shen_ICRA2018}, and spheres \cite{ding_gao_wang_shen_TRO2019, gao_shen_SSR2016, gao_etal_JFR2019}, with the aim of maximizing their volume to limit the number of safe corridors and, consequently, the B\'ezier curves used in optimization.
Piecewise-linear reference paths are commonly preferred due to their simple construction, e.g., using graph search methods \cite{chen_su_shen_ROBIO2015, liu_etal_RAL2017, chen_liu_shen_ICRA2016} or sampling-based planning algorithms \cite{gao_shen_SSR2016, gao_etal_JFR2019, wang_etal_ICAUS2022}. 
Fast marching with a distance field is applied to find a reference path that balances the distance to obstacles, facilitating the creation of maximal volume safe corridors  \cite{lin_etal_JFR2018, gao_wu_lin_shen_ICRA2018}.
In this paper, to achieve a minimal number of safe corridors and B\'ezier curves, we iteratively construct maximal-volume convex polygonal safe corridors along a maximal-clearance minimal-length reference path using the separating hyperplanes of corridor centers from obstacles \cite{liu_etal_RAL2017, arslan_koditschek_IJRR2019, arslan_pacelli_koditschek_IROS2017}
where we find a maximal-clearance minimal-length reference path over a binary occupancy map using \mbox{graph search with the inverse distance field as a cost map.}

\subsection{Contributions and Organization of the Paper}
\label{sec.Contributions}

In this paper, we introduce a new family of graph-theoretic quadratic B\'ezier optimization objectives that use a consensus distance to measure the alignment between B\'ezier control points based on the Laplacian of their weighted interaction graph.
In \refsec{sec.BezierCurves}, we provide a brief background on B\'ezier curves and highlight their important analytical properties that enable an explicit matrix formulation of B\'ezier curve optimization.
In \refsec{sec.BezierOptimizationObjective}, we present a new quadratic consensus distance for B\'ezier control points and provide physical, geometric, and statistical consensus distance examples.
In \refsec{sec.SafeCorridors}, using separating hyperplanes of obstacles, we describe a systematic and effective way of automatically constructing maximal-volume convex polygonal corridors along a reference path that balances distance to obstacles. 
In \refsec{sec.NumericalSimulations}, we present example applications of B\'ezier curve optimization over safe corridors using various quadratic optimization objectives.   
We conclude in \refsec{sec.Conclusions} with a summary of our contributions and future work.

\section{B\'ezier Curves}
\label{sec.BezierCurves}

In this section, we introduce the notation used throughout the paper and provide a brief background on B\'ezier curves, including their important properties and matrix operations, which are essential to B\'ezier curve optimization in safe and smooth motion planning.   
\begin{definition}\label{def.BezierCurve}
\emph{(B\'ezier Curve)} In a $\bdim$-dimensional Euclidean space $\R^{\bdim}$, a \emph{B\'ezier curve} $\bcurve_{\bpoint_0, \ldots \bpoint_{\bdegree}} (t)$ of degree $\bdegree \in \N$, associated with \emph{control points} $\bpoint_0, \ldots, \bpoint_\bdegree \in \R^{\bdim}$,  is  a parametric $\bdegree^{\text{th}}$-order polynomial curve defined for $0 \leq t \leq 1$ as
\begin{align} \label{eq.BezierCurve}
\bcurve_{\bpoint_0, \ldots, \bpoint_\bdegree}(t) \ldf \sum_{i=0}^\bdegree \bpoly_{i,\bdegree}(t) \bpoint_i, 
\end{align}
where  $\bpoly_{i,\bdegree}(t)$ denotes the $i^{\text{th}}$ Bernstein basis polynomial of degree $\bdegree$ that is  defined for $ i = 0,1, \ldots, \bdegree$ as
\begin{align}\label{eq.BernsteinPolynomial}
\bpoly_{i,\bdegree}(t) \ldf \scalebox{1.2}{$\binom{n}{i}$} t^i (1-t)^{n-i}.
\end{align}
\end{definition}

To effectively work with high-order B\'ezier curves with a large number of control points, it is convenient to use the matrix representation of B\'ezier curves in the form of
\begin{align}
\bcurve_{\bpmat}(t) =    \tr{\bpmat} \bbasis_{\cdegree}(t), 
\end{align}
where $\bpmat \! := \! \tr{\blist{\bpoint_0, \ldots, \bpoint_n}} \! \in\! \R^{(\bdegree+1) \times \bdim}$ denotes  the control point matrix, $\bbasis_{\bdegree}(t) := \tr{\blist{\bpoly_{0,\bdegree}(t), \ldots, \bpoly_{\bdegree,\bdegree}(t)}} \in \R^{\bdegree+1 }$ is the Bernstein basis vector of degree $\bdegree$, and $\tr{\mat{A}}$ denotes the transpose of a matrix $\mat{A}$.

Some key characteristics of B\'ezier and Bernstein polynomials that are important for motion planning  are their interpolation, derivative, and convexity  properties \cite{farouki_CAGD2012, farin_CurvesSurfaces2002, arslan_tiemessen_TRO2022}.

\begin{property} \label{propty.BezierInterpolation}
\emph{(Interpolation)} A B\'ezier curve, $\bcurve_{\pmat}(t)$ associated with control points $\bpmat = \tr{\blist{\bpoint_0, \ldots, \bpoint_\bdegree}}$, smoothly interpolates between its first and last control points, i.e.,
\begin{align}
\bcurve_{\bpmat}(0) &=  \tr{\bpmat} \bbasis_{\bdegree} (0)  = \bpoint_0 
\\
\bcurve_{\bpmat}(1) &= \tr{\bpmat} \bbasis_{\bdegree} (1) = \bpoint_\bdegree 
\end{align}
since Bernstein basis vector \mbox{$\bbasis_{\bdegree}(t) = \tr{\blist{\bpoly_{0,\bdegree}(t), \ldots, \bpoly_{\bdegree,\bdegree}(t)}}$} smoothly interpolates between 
\begin{align}
\bbasis_{\bdegree} (0) = \scalebox{0.8}{$\begin{bmatrix}
1 \\ 0\\ \vdots\\ 0
\end{bmatrix}$} 
\quad  \text{ and } \quad   
\bbasis_{\bdegree} (1)  = \scalebox{0.8}{$\begin{bmatrix}
0\\ \vdots\\ 0\\1
\end{bmatrix}$}.
\end{align}
\end{property}
\noindent The interpolation property of B\'ezier curves allows for describing endpoint constraints, such as start and goal positions, as a linear equality constraint in B\'ezier curve optimization, as illustrated in \reftab{tab.BezierCurveOptimization}.

\begin{property}\label{propty.BezierDerivative}
\emph{(High-Order Derivatives)} The $k^{\text{th}}$-order derivative of a B\'ezier curve of degree $n$, where $k \in \clist{0 \ldots, n}$, is another B\'ezier curve   of degree $(n-k)$ that is given by\footnote{In addition to demonstrating the relationship between continuous differentiation and discrete differencing, the matrix representation of $k^{\text{th}}$-order derivative of B\'ezier curves offers a compact alternative for rewriting the standard derivative form 
$\frac{\diff^{k}}{\diff t ^k} \bcurve_{\bpoint_0, \ldots, \bpoint_{\bdegree}}(t) = \tfrac{n!}{(n - k)!}  \bcurve_{\qpoint_0, \ldots, \qpoint_{n-k}}(t)$ where $\qpoint_i = \sum_{j=0}^{k} \binom{k}{j} (-1)^{k-j} \bpoint_{i+j}$ for $i = 0, \ldots, n-k$ \cite{farin_CurvesSurfaces2002};
similarly, the $k^{\text{th}}$-order derivative of Bernstein basis polynomials is often written as
 $\frac{\diff^k}{\diff t ^k} \bpoly_{i,n}(t) \!=\! \frac{n!}{(n-k)!} \sum\limits_{j=0}^{k} \binom{k}{j} (-1)^{k-j} \bpoly_{i-j, n-k}(t)$ for $i = 0, \ldots, n$ \cite{farin_CurvesSurfaces2002}. 
} 
\begin{align}
\bcurve^{(k)}_{\pmat}(t):= \frac{\diff^{k}}{\diff t ^k} \bcurve_{\pmat}(t) &= \scalebox{0.85}{$\dfrac{n!}{(n - k)!}$}  \bcurve_{\Dmat(n, k) \pmat}(t) 
\end{align}
where $\mat{D}(n,k)$ denotes the $k^{\text{th}}$-order forward finite-difference matrix in $\R^{(n-k+1) \times (n+1)}$ whose elements  are 
\begin{align}\label{eq.HighOrderFiniteDifference}
\blist{\mat{D}(n,k)}_{i+1, j+1} \!=\! \left \{
\begin{array}{@{}c@{}l@{}}
\binom{k}{j-i}\!(-1)^{k - j+ i} & \text{, if } 0 \leq j-i \leq k 
\\
0 & \text{, otherwise.}
\end{array}
\right . \!
\end{align}
for $i = 0, \ldots, (n-k)$ and $j = 0, \ldots, n$, and the $k^{\text{th}}$-order derivative of the Bernstein  basis vector $\bbasis_{n}(t)$ also satisfies
\begin{align}
\frac{\diff^k}{\diff t ^k} \bbasis_{n}(t) = \tfrac{n!}{(n-k)!} \tr{\Dmat(n,k)} \bbasis_{n-k}(t).
\end{align}
\end{property}

\noindent Note that, by definition, $\Dmat(n,0) = \mat{I}$ and $\Dmat(n,k) \mat{1} = \mat{0}$ for any $k=1, \ldots, n$, where $\mat{I}$ denotes the square identity matrix; and $\mat{1}$ and $\mat{0}$ are, respectively,  the column vector of all ones and zeros of appropriate sizes.
If more precision is required, we use subscripts to specify the dimensions of these matrices, for example,  $\mat{I}_{n}$ and $\mat{I}_{n \times n}$ denote the $n \times n$ identity matrix whereas $\mat{0}_{n}$ and $\mat{0}_{n \times 1}$ denote the $n \times 1$ zero column vector.   
For the special case of the first-order B\'ezier derivative $\frac{\diff}{\diff t} \bcurve_{\pmat} (t) = n \bcurve_{\mat{D}(n,1)\pmat} (t) $, 
\begin{align}
\mat{D}(n,1) = \blist{\mat{0}_{n\times 1}, \mat{I}_{n \times n}} - \blist{\mat{I}_{n \times n}, \mat{0}_{n\times 1}} 
\end{align}
which allows us to recursively determine the $k^{\text{th}}$-order forward difference matrix $\Dmat(n,k)$ as
\begin{align}
\mat{D}(n,k) = \mat{D}(n-1, k-1) \mat{D}(n, 1).
\end{align}

It is important to highlight that the high-order derivative property of B\'ezier curves enables the expression of endpoint derivative constraints (e.g., ensuring a specific level of continuity and smoothness between adjacent B\'ezier curve segments as shown in \reftab{tab.BezierCurveOptimization}) as a linear equality constraint of B\'ezier control points as
\begin{align}
\frac{\diff^{k}}{\diff t^{k}}\bcurve_{\pmat}(0) &=  \tfrac{n!}{(n-k)!}\tr{\pmat}\tr{\mat{D}(n,k)} \bbasis_{n-k}(0)   \\
\frac{\diff^{k}}{\diff t^{k}}\bcurve_{\pmat}(1) &=  \tfrac{n!}{(n-k)!}\tr{\pmat}\tr{\mat{D}(n,k)} \bbasis_{n-k}(1).
\end{align}

\begin{table*}[t]
\caption{An Explicit Matrix Formulation for Quadratic Optimization of B\'ezier Curves}
\label{tab.BezierCurveOptimization}
\begin{center}
\vspace{-5mm}
\centering
\begin{tabular}{llcl}
\hline
\hline \\[-1mm]
\multicolumn{4}{c}{\parbox{0.91\textwidth}{Find the control point matrices $\pmat_1, \ldots, \pmat_m \in \R^{(n+1) \times d}$ of $m$ B\'ezier curves of degree $n$ that join $\vect{p}_{\mathrm{start}}, \vect{p}_{\mathrm{goal}} \in \R^{d}$  and minimize the total squared norm of  the $k^{\text{th}}$-order  derivative of B\'ezier curves with a certain degree $C$ of continuity under given linear inequality constraints}}
\vspace{1mm}
\\
\hline
\\
minimize & $\sum\limits_{i=1}^{m} \int_{0}^{1} \norm{\frac{\diff^k}{\diff t^k} \bcurve_{\pmat_i}(t)}^2 \diff t$ & $\Longleftrightarrow$ & $\sum\limits_{i=1}^{m}\trace\plist{\tr{\pmat_i} \tr{\Dmat(n,k)} \HmatN(n-k) \Dmat(n,k) \pmat_i}$
\\
subject to  & $\bcurve_{\pmat_1}(0) = \vect{p}_{\mathrm{start}}$, $\bcurve_{\pmat_m}(1) = \vect{p}_{\mathrm{goal}}$ & $\Longleftrightarrow$& $\tr{\pmat_1} 
\scalebox{0.6}{$\begin{bmatrix} 1 \\ 0 \\ \vdots \\ 0\end{bmatrix}$} 
= \vect{p}_{\mathrm{start}}$, $\tr{\pmat_m} 
\scalebox{0.6}{$\begin{bmatrix}0 \\ \vdots\\ 0\\ 1 \end{bmatrix}$} 
= \vect{p}_{\mathrm{goal}}$ \\
& $\frac{\diff^c}{\diff t^c} \bcurve_{\pmat_i}(1) = \frac{\diff^c}{\diff t^c} \bcurve_{\pmat_{i+1}}(0) \quad$ \parbox{2.4cm}{$\forall i = 1, \dots, m-1$ \\ $\forall c = 0, \ldots, C $} & $\Longleftrightarrow$ & $\tr{\pmat_i}\tr{\Dmat(n,c)} \scalebox{0.6}{$\begin{bmatrix}0 \\ \vdots\\ 0\\ 1 \end{bmatrix}$} 
 = \tr{\pmat_{i+1}}\tr{\Dmat(n,c)} \scalebox{0.6}{$\begin{bmatrix}1 \\ 0 \\ \vdots\\ 0 \end{bmatrix}$} \quad $ 
  \parbox{2.4cm}{$\forall i = 1, \dots, m-1$ \\ $\forall c = 0, \ldots, C $}
\\
\\
& $\mat{A}_i \bcurve_{\pmat_i}(t) \leq \vect{b}_i  \quad \forall t \in [0,1], \forall i = 1, \dots, m-1$ & $\Longleftarrow$ & $\mat{A}_i \tr{\pmat_i} \leq \vect{b}_i \tr{\mat{1}} \quad \forall i = 1, \dots, m-1$
\\[+1mm]
\hline
\hline
\end{tabular}
\end{center}
\end{table*}

\begin{property} \label{propty.BezierConvexity}
\emph{(Convexity)} A B\'ezier curve $\bcurve_{\pmat}(t)$ is defined as a continuous convex combination of its control points $\pmat= \tr{\blist{\bpoint_0, \ldots, \bpoint_\bdegree}}$, and so it is contained in the convex hull, denoted by $\conv$, of the control points, i.e.,
\begin{align}
\bcurve_{\bpmat}(t) \in \conv\plist{\bpoint_0, \ldots, \bpoint_\bdegree} \quad \quad \forall t \in [0,1], 
\end{align}
because Bernstein polynomials are nonnegative and sum to one, i.e., for any $t \in [0,1]$ 
\begin{align}
\bpoly_{i,n}(t) \geq 0, \quad \text{and} \quad \sum_{i=0}^{n} \bpoly_{i,n}(t) = 1. 
\end{align}
\end{property}

\noindent The convexity of B\'ezier curves allows for a simple over-approximation of linear inequality (e.g., safety) constraints of the form $\mat{A} \bcurve_{\pmat}(t) \leq \vect{b}$ by $\mat{A}\tr{\pmat} \leq \vect{b} \tr{\mat{1}}$ in B\'ezier curve optimization as seen in \reftab{tab.BezierCurveOptimization} because the convexity property of B\'ezier curves ensures that
\begin{align}
\mat{A}\tr{\pmat} \leq \vect{b} \tr{\mat{1}} \Longrightarrow \mat{A} \bcurve_{\pmat}(t) \leq \vect{b} \quad \forall t \in [0,1].
\end{align}

Another crucial property of B\'ezier curves for defining quadratic optimization objectives is that the inner product of B\'ezier curves can be explicitly calculated as the inner product of their control points, which is mainly due to the constant definite integral property of Bernstein polynomials (i.e.,$\int_{0}^{1} \bpoly_{k,\bdegree} (t) \diff t = \frac{1}{\bdegree + 1} \quad \forall k = 0, \ldots, \bdegree$).
\begin{lemma}\label{lem.BezierInnerProduct}
\emph{(B\'ezier Inner Product)}
The inner product of a pair of B\'ezier curves $\bcurve_{\pmat}(t)$ and $\bcurve_{\qmat}(t)$ with associated control points \mbox{$\pmat = \tr{[\ppoint_0, \ldots, \ppoint_n]} \!\in\! \R^{(n+1) \times d}$}  and \mbox{$\qmat = \tr{[\qpoint_0, \ldots, \qpoint_m]} \!\in\! \R^{(m+1)\times d}$} satisfies
\begin{align}
\int_{0}^{1} \tr{\bcurve_{\pmat}(t)} \bcurve_{\qmat}(t) \diff t = \trace\plist{\pmat \HmatB(n,m) \tr{\qmat}}
\end{align}
where $\trace$ denotes the matrix trace operator and $\HmatB(n,m) := \int_{0}^{1} \bbasis_{n}(t) \tr{\bbasis_{m}(t)} \diff t$ is the Bernstein outer product matrix in $\R^{(n+1) \times (m+1)}$ whose elements are given for $i = 0, \ldots, n$ and $j = 0, \ldots, m$ by 
\begin{align}\label{eq.BezierInnerProductHessian}
\blist{\HmatB(n,m)}_{i+1, j+1} = \tfrac{1}{m+n+1} \frac{\binom{n}{i} \binom{m}{j}}{\binom{m+n}{i+j}}.  
\end{align}
\end{lemma}
\begin{proof}
See \refapp{app.BezierInnerProduct}.
\end{proof}

The inner product property of B\'ezier curves is crucial for recognizing that the squared norm of a B\'ezier curve of degree $n$  is a convex quadratic function of its control points, i.e., 
\begin{align}
\norm{\bcurve_{\bpmat}}^2 := \int_{0}^{1} \norm{\bcurve_{\pmat}(t)}^2 \diff t = \trace \plist{\tr{\pmat} \HmatN(n)\pmat} 
\end{align}
where the elements of the Hessian $\HmatN(n):=\HmatB(n,n)$ of the squared B\'ezier norm are given for $i,j = 0, \ldots, n$ by
\begin{align}\label{eq.BezierNormHessian}
\blist{\HmatN(n)}_{i+1, j+1} = \tfrac{1}{2n+1} \frac{\binom{n}{i} \binom{n}{j}}{\binom{2n}{i+j}}.  
\end{align}
Note that   $\mat{0}\!\preceq\! \HmatN(n) \!\preceq \! \frac{1}{n+1} \mat{I}$ due to Jensen's inequality.\footnote{\label{fn.BezierNormHessianBound} One can verify $\mat{0} \preceq \HmatN(n) \preceq \frac{1}{n+1} \mat{I}$  using the convexity of Bernstein polynomials (\refpropty{propty.BezierConvexity}),  their integral property $\int_{0}^{1} \bpoly_{n}(t) \diff t = \frac{1}{n+1} \mat{1}$, and Jensen's inequality as follows:  for any vector $\vect{v} \!=\! \tr{\blist{v_0, \ldots, v_n}} \!\!\in \!\R^{n+1}$
\begin{align*}
\tr{\vect{v}} \HmatN(n) \vect{v} &= \tr{\vect{v}} \plist{\int_{0}^{1} \bbasis_{n}(t) \tr{\bbasis_{n}(t)} \diff t} \vect{v} =  \int_{0}^{1} (\tr{\bbasis_{n}(t)}\vect{v})^2  \diff t \\
& \leq \int_{0}^{1} \tr{\bbasis_{n}(t)} \scalebox{0.8}{$\begin{bmatrix} v_0^2 \\ \vdots \\ v_n^2 \end{bmatrix}$} \diff t  = \tfrac{1}{n+1}\norm{\vect{v}}^2 
\end{align*} 
where the nonnegativity is due to the squared form. 
}
The analytical form of the Hessian of the B\'ezier derivative norm, combined with B\'ezier matrix operations, enables an explicit and compact formulation for quadratic optimization of B\'ezier curves as shown in \reftab{tab.BezierCurveOptimization}.

Yet another quadratic function of B\'ezier curves that is crucial for B\'ezier curve optimization is variance. 

\begin{lemma}\label{lem.BezierStatistics}
\emph{(B\'ezier Statistics)}
The mean of a B\'ezier curve $\bcurve_{\bpmat}(t)$ and the mean of its controls points $\bpmat=\tr{\blist{\bpoint_0, \ldots, \bpoint_{n}}} $ are the same, whereas the variance of the B\'ezier curve is tightly bounded above by the variance of the control points~as
\begin{align}
\bmean(\bcurve_{\bpmat}) = \bmean(\bpmat) \, \text{ and }\, \bvar(\bcurve_{\bpmat}) \leq \tfrac{1}{n+1} \bvar(\bpmat)
\end{align}
where the overloaded mean and variance operators for both B\'ezier curves and B\'ezier control points are defined as\reffn{fn.ControlPointVariance}
\begin{align}
\!\!\bmean(\bcurve_{\bpmat}) &\!:= \!\!\int\limits_{0}^{1} \!\bcurve_{\bpmat}(t) \diff t,  \!\!\!
 & 
\!\!\!\bvar(\bcurve_{\bpmat}) &\!:=\!  \int\limits_{0}^{1} \norm{\bcurve_{\bpmat}(t) - \bmean(\bcurve_{\bpmat})}^2 \diff t  \nonumber
\\
& & &\!\! =\trace\plist{\tr{\bpmat} \Smat(n) \HmatN(n) \Smat(n) \bpmat},  \nonumber
\\ 
\!\!\bmean(\bpmat) &\!:=\!  \tfrac{1}{n+1}\! \sum_{i=0}^{n} \bpoint_i, 
& 
\bvar(\bpmat) &\!:= \!\tfrac{1}{n+1}\sum_{i=0}^{n} \norm{\bpoint_i - \bpmean(\bcurve_{\bpmat})}^2 \nonumber
\\
& 
& & =  \tfrac{1}{n+1} \trace\plist{\tr{\bpmat}\Smat(n) \bpmat} \label{eq.BezierStatistics}
\end{align}
in terms of the B\'ezier norm Hessian $\HmatN(n)$ in \refeq{eq.BezierNormHessian} and  the mean-shift matrix $\Smat(n)$ that is defined to be
\begin{align}\label{eq.MeanShiftMatrix}
\Smat(n) := \mat{I}_{n+1} - \tfrac{1}{n+1} \mat{1}_{n+1} \tr{\mat{1}}_{n+1}.
\end{align}
\end{lemma}
\begin{proof}
See \refapp{app.BezierStatistics}.
\end{proof}
\noindent It is useful to observe that the mean-shift matrix is a projection operator, i.e.,  $\Smat(n)\Smat(n) = \Smat(n)$, that is symmetric and positive semidefinite, and it satisfies $\Smat(n) \mat{1} = \mat{0}$.

\addtocounter{footnote}{1}
\footnotetext{\label{fn.ControlPointVariance}The variance of B\'ezier control points can be equivalently expressed as
\begin{align*}
\bvar\plist{\pmat} &= \tfrac{1}{(n+1)}\trace\plist{\tr{\pmat}\Smat(n)\pmat}
= \tfrac{1}{n+1}\!\sum_{i=0}^{n} \norm{\bpoint_i \!-\! \tfrac{1}{n+1}\!\sum_{j=0}^{n} \bpoint_j}^2
\\
 & = \tfrac{1}{(n+1)^2} \! \sum_{i=0}^{n} \sum_{j = 0}^{n} \tfrac{1}{2}\norm{\ppoint_i \!-\! \ppoint_j}^2. 
\end{align*}  
 }

\section{\!\!\scalebox{0.97}{\mbox{Graph-Theoretic B\'ezier Optimization Objective\!}}}
\label{sec.BezierOptimizationObjective}

In this section, we introduce a new consensus-based quadratic B\'ezier optimization objective that is designed to assess the agreement and alignment of B\'ezier control points based on the Laplacian matrix of their interaction graph.
We present example physical, geometric, and statistical consensus-based optimization objectives that are constructed based on differential norms and variances.

\subsection{Consensus-Based B\'ezier Optimization Objective}

A B\'ezier curve is, by definition, constructed based on a continuous convex combination of its control points to smoothly interpolate and connect its first and last control points (see \refdef{def.BezierCurve}, \refpropty{propty.BezierInterpolation}, and \refpropty{propty.BezierConvexity}). 
Hence, the shape of a B\'ezier curve (whether it is straight or oscillatory) is completely determined by its control points.
A smooth and high-quality B\'ezier curve that satisfies a set of given geometric (e.g., boundary, length, intersection) conditions requires a certain level of agreement or alignment between its control points.
Accordingly, inspired by consensus in networked multi-agent systems \cite{olfati_fax_murray_IEEE2007}, to determine the agreement of control points of a B\'ezier curve, we assume that B\'ezier control points interact over a weighted undirected connected graph (without self loops) specified\footnote{We find it convenient to specify an interaction graph using a graph Laplacian matrix instead of weighted adjacency matrix because it is usually difficult to ensure the positive semidefiniteness of graph Laplacian for negatively weighted graphs.} by a positive semidefinite\footnote{$\PSDM^{n}$ is the set of symmetric and positive semidefinite matrices in $\R^{n \times n}$.} Laplacian matrix $\mat{L} \in \PSDM^{n+1}$  with $\mat{L} \mat{1} = 0$, where the interaction weights $\mat{W} \in \R^{(n+1) \times (n+1)}$ between the control points for $i,j = 0, \ldots, n$ are given by
\begin{align}
\blist{\mat{W}}_{i+1, j+1} = \left \{ 
\begin{array}{@{}cl@{}}
0 & \text{ if } i = j \\
-[\mat{L}]_{i+1,j+1} & \text{ otherwise.}
\end{array}
\right.
\end{align}
Note that an interaction weight between a pair of control points might be positive or negative. A positive interaction weight signifies similarity among control points, whereas a negative interaction weight promotes divergence or dissimilarity among control points.
Accordingly, as a new unifying B\'ezier curve optimization objective, we consider a common measure of distance from consensus \cite{olfati_fax_murray_IEEE2007} that is defined based on the interaction graph Laplacian as follows.

\begin{definition} \label{def.ConsensusDistance}
\emph{(B\'ezier Consensus Distance)}
Given a symmetric and positive semidefinite interaction graph Laplacian $\mat{L} \in \PSDM^{\bdegree+1}$ with $\mat{L} \mat{1} = \mat{0}$, the \emph{consensus distance}, denoted by $\bconsdist_{\mat{L}}$, of the control points $\pmat \in \R^{(\bdegree +1)\times \bdim}$ of a B\'ezier curve $\bcurve_{\bpmat}(t)$ is a measure of disagreement of the control points that is quantified as $\bconsdist_{\mat{L}}(\pmat):= \trace\plist{\tr{\mat{P}} \mat{L} \mat{P}}$.
\end{definition}
 
\noindent The singularity of graph Laplacian, i.e., $\mat{L} \mat{1} \! =\! \mat{0}$, implies that the consensus distance of B\'ezier control points is determined by the relative differences of the control points rather than absolute control point locations.
For a connected interaction graph with positive weights (i.e., the graph Laplacian has exactly one zero-eigenvalue), the perfect consensus of B\'ezier control points (i.e., $\trace\plist{\tr{\pmat} \mat{L} \pmat} = 0$) means a constant B\'ezier curve (i.e., $\bcurve_{\pmat}(t) = \ppoint$ since $\pmat = \tr{\blist{\ppoint, \ldots, \ppoint}}$), which corresponds to zero motion at $\ppoint$.  
In \reftab{tab.QuadraticBezierOptimizationObjectives}, we present example physical, geometric, and statistical consensus distances for B\'ezier curve optimization whose graph-theoretic interpretation is illustrated  in \reftab{tab.GraphTheoreticOptimizatinoObjectives} and discussed below.

\subsection{Physical B\'ezier Optimization Objectives}

In polynomial trajectory optimization, a widely-used physical optimization objective is the squared  norm of B\'ezier derivatives, for example, to minimize total velocity, acceleration, jerk, snap and their various combinations.   

\begin{proposition}\label{prop.BezierDerivativeNorm}
\emph{(B\'ezier Derivative Norm)}
The squared norm of $k^{\text{th}}$-order derivative of an $n^{\text{th}}$-order B\'ezier curve associated with control point matrix $\pmat \in \R^{(n+1)\times \cdim}$ is given by
\begin{align}\label{eq.BezierDerivativeNorm}
\int_{0}^{1} \norm{\tfrac{\diff^k}{\diff t^k} \bcurve_{\pmat}(t)}^2 \diff t = \frac{n!^2}{(n-k)!^2}\trace\plist{\tr{\pmat} \HmatDN(n,k) \pmat} 
\end{align}   
where  the Hessian $\HmatDN(n,k)$ of B\'ezier derivative norm
\begin{align}\label{eq.BezierDerivativeNormHessian}
\HmatDN(n,k) := \tr{\Dmat(n,k)} \HmatN(n - k) \Dmat(n,k)
\end{align}
is  determined by the high-order finite-difference matrix $\Dmat(n,k)$ in \refeq{eq.HighOrderFiniteDifference} and the B\'ezier norm Hessian $\HmatN(n)$ in \refeq{eq.BezierNormHessian}.  
\end{proposition}
\begin{proof}
The statement is a consequence of the derivative and inner-product properties of B\'ezier curves in \refpropty{propty.BezierDerivative} and \reflem{lem.BezierInnerProduct}, respectively, i.e.,  $\tfrac{\diff^k}{\diff t^k} \bcurve_{\pmat}(t) = \frac{n!}{(n-k)!}\bcurve_{\Dmat(n,k)\pmat} (t)$ and   $\int_{0}^{1} \norm{\bcurve_{\pmat}(t)}^2 \diff t = \trace\plist{\tr{\pmat} \HmatN(n) \pmat}$.
\end{proof}

\noindent Note that $\HmatDN(n,k) = \tr{\Dmat(n,k)} \HmatN(n - k) \Dmat(n,k)$ is a symmetric and positive semidefinite matrix and satisfies $\HmatDN(n,k) \mat{1} = \mat{0}$  for $k = 1, \ldots, n$ because $\HmatN(n-k)$ is positive definite and $\Dmat(n,k)\mat{1} = \mat{0}$.
Hence, from a graph-theoretic perspective, one may consider the Hessian $\HmatDN(n,k)$ of B\'ezier derivative norm as a graph Laplacian matrix to better understand the interaction relation between B\'ezier control points.
In \reftab{tab.GraphTheoreticOptimizatinoObjectives}, we illustrate the weighted interaction graphs of B\'ezier control points corresponding to B\'ezier velocity and acceleration norms using the associated Hessian matrices $\HmatDN(n,1)$ and $\HmatDN(n,2)$ for $n = 2, 3, 4$.

\begin{table}[t]
\caption{Example Consensus-Based B\'ezier Optimization Objectives}
\label{tab.QuadraticBezierOptimizationObjectives}
\vspace{-2mm}
\begin{tabular}{@{}c@{}c@{}}
\hline
\hline \\[-2mm]
Objection Function & Consensus Distance \vspace{1mm}\\
\hline
\\[-2mm]
\parbox{2.3cm}{\centering Control-Point \\ Difference Norm \\ $\norm{\Dmat(n,k) \pmat}^2$} & $\trace\plist{\tr{\pmat} \tr{\Dmat(n,k)}  \Dmat(n,k) \pmat}$ 
\\
\\
\parbox{2.3cm}{\centering B\'ezier Derivative \\ Norm \\ $\norm{\bcurve_{\Dmat(n,k) \pmat}}^2$} & $\trace\plist{\tr{\pmat} \tr{\Dmat(n,k)} \HmatN(n\!-\!k)  \Dmat(n,k) \pmat}$
\\
\\
\parbox{2.3cm}{\centering Control-Point \\ Difference Variance \\ $\bvar\plist{\Dmat(n,k) \pmat}$} & $\trace\plist{\tr{\pmat} \tr{\Dmat(n,k)} \Smat(n\!-\!k) \Dmat(n,k) \pmat}$ 
\\
\\
\parbox{2.3cm}{\centering B\'ezier Derivative Variance \\ $\bvar\plist{\bcurve_{\Dmat(n,k) \pmat}}$} & $\hspace{-0.5mm}\trace\plist{\!\tr{\pmat}  \tr{\Dmat(n,k)\!} \Smat(n\!-\!k)\HmatN(n\!-\!k)\Smat(n\!-\!k)  \Dmat(n,k) \pmat\!}\!$\hspace{-1mm}
\vspace{1mm}
\\
\hline
\hline
\end{tabular}
\end{table}

\subsection{Geometric B\'ezier Optimization Objectives}

Replacing continuous differentiation with finite differencing often yields geometrically more intuitive  B\'ezier optimization objectives.
More specifically, instead of B\'ezier derivative norm $\int_{0}^{1}\norm{\frac{\diff^k}{\diff t^k} \bcurve_{\pmat}(t)} \diff t$, as a geometric optimization objective, we consider the squared norm of finite difference of B\'ezier control points that is given by 
\begin{align}\label{eq.BezierDifferenceNorm}
\norm{\Dmat(n,k) \pmat}^2 = \trace\plist{\tr{\pmat} \HmatDDN(n,k) \pmat}
\end{align}
where the Hessian $\HmatDDN(n,k)$ of B\'ezier control-point difference norm is defined in terms of the finite-difference matrix $\Dmat(n,k)$ in \refeq{eq.HighOrderFiniteDifference} as
\begin{align}
\HmatDDN(n,k) := \tr{\Dmat(n,k)} \Dmat(n,k).
\end{align}
In \reftab{tab.GraphTheoreticOptimizatinoObjectives}, we present the corresponding interaction graphs of the first-order and second-order difference norm of B\'ezier control points.
As seen in \reftab{tab.GraphTheoreticOptimizatinoObjectives}, as a special case of $n = 2$, the second-order difference and derivative norms define the same interaction relation between the control points. 
However, the finite-difference norm of B\'ezier control points often results in reduced interaction between control points compared to the B\'ezier derivative norm, which allows for an intuitive interpretation of geometric optimization objectives.   
Below, we present the explicit forms of the norm of the first- and second-order differences of control points to illustrate the geometric intuition behind the corresponding B\'ezier optimization objectives.

The first-order difference norm of B\'ezier control points defines an upper bound on the squared length of the B\'ezier (control) polygon that connects B\'ezier control points with straight lines from the first control point to the last control point, i.e.,
\begin{align*}
 \norm{\Dmat(n,1) \pmat}^2 
 &= \sum_{k=1}^{n} \norm{\bpoint_{k}\! -\! \bpoint_{k-1}}^2 \geq \frac{1}{n} \plist{\sum_{k=1}^{n}\norm{\bpoint_{k}\! -\! \bpoint_{k-1}}\!}^{\!2}
\end{align*}     
where the inequality follows from Jensen's inequality.

The second-order B\'ezier difference norm measures the total quadratic inequality gap between three consecutive B\'ezier control points, i.e.,
\begin{align*}
\norm{\Dmat(n,k) \pmat}^2 
& = \sum_{i=1}^{n-1} \norm{\ppoint_{i+1} \!- \!2\ppoint_{i}\! +\! \ppoint_{i-1}}^2 \\
& \hspace{-15mm} = \sum_{i=1}^{n-1} 2 \norm{\ppoint_{i+1} \!-\! \ppoint_{i}}^2 \!+\! 2\norm{\ppoint_i \!-\! \ppoint_{i-1}}^2 - \norm{\ppoint_{i+1} \!-\! \ppoint_{i-1}}^2 
\end{align*} 
where $\norm{\ppoint_{i+1} - \ppoint_{i-1}}^2 \leq 2 \norm{\ppoint_{i+1} - \ppoint_{i}}^2 + 2\norm{\ppoint_i - \ppoint_{i-1}}^2$ is known as the quadratic inequality that is tight if and only if  $\bpoint_k$ is the midpoint of $\bpoint_{k+1}$ and $\bpoint_{k-1}$.
Hence, the second-order B\'ezier difference norm measures the homogeneity (i.e., the linearity and uniformity) of control points and 
aims at uniformly placing control points along a straight line.

\begin{table}[t]
\caption{Graph-Theoretic Optimization Objectives \\ and Interaction Graphs\reffn{fn.InteractionGraphWeights} of B\'ezier Control Points}
\label{tab.GraphTheoreticOptimizatinoObjectives} 
\center
\vspace{-4mm}
\begin{tabular}{c@{}c@{\hspace{1mm}}c@{\hspace{1mm}}c}
\hline \hline \\[-2mm]
\begin{tabular}{@{}c@{}}
First-Order \\ Derivative Norm  \\
\scalebox{0.8}{$\trace\plist{\tr{\pmat} \HmatDN(n,1) \pmat}$}
\end{tabular} 
& 
\begin{tabular}{@{}c@{}}
\scalebox{0.45}{\input{drawings/tikz_first_order_continuous_derivative_n2}} 
\end{tabular}
& 
\begin{tabular}{@{}c@{}}
\scalebox{0.45}{\input{drawings/tikz_first_order_continuous_derivative_n3}} 
\end{tabular}
&
\begin{tabular}{@{}c@{}}
\scalebox{0.45}{\input{drawings/tikz_first_order_continuous_derivative_n4}}
\end{tabular}
\\
\begin{tabular}{@{}c@{}}
First-Order \\ Difference Norm  \\
\scalebox{0.8}{$\sum\limits_{i=1}^{n} \norm{\bpoint_i - \bpoint_{i-1}}^2$}\\
\scalebox{0.8}{$\trace\plist{\tr{\pmat} \HmatDDN(n,1) \pmat}$}
\end{tabular} 
& 
\begin{tabular}{@{}c@{}}
\scalebox{0.45}{\input{drawings/tikz_first_order_discrete_difference_n2}} 
\end{tabular}
& 
\begin{tabular}{@{}c@{}}
\scalebox{0.45}{\input{drawings/tikz_first_order_discrete_difference_n3}} 
\end{tabular}
&
\begin{tabular}{@{}c@{}}
\scalebox{0.45}{\input{drawings/tikz_first_order_discrete_difference_n4}}
\end{tabular}
\\
\begin{tabular}{@{}c@{}}
Zeroth-Order\\
Difference Variance  \\
\scalebox{0.8}{$\sum\limits_{i=0}^{n} \sum\limits_{j=0}^{n} \frac{1}{2} \norm{\bpoint_i - \bpoint_{j}}^2$}\\
\scalebox{0.8}{$\trace\plist{\tr{\pmat} \HmatDDV(n,0) \pmat}$}
\end{tabular} 
& 
\begin{tabular}{@{}c@{}}
\scalebox{0.45}{\input{drawings/tikz_control_point_variance_n2}} 
\end{tabular}
& 
\begin{tabular}{@{}c@{}}
\scalebox{0.45}{\input{drawings/tikz_control_point_variance_n3}} 
\end{tabular}
&
\begin{tabular}{@{}c@{}}
\scalebox{0.45}{\input{drawings/tikz_control_point_variance_n4}}
\end{tabular}
\\
\begin{tabular}{@{}c@{}}
Second-Order \\ Derivative Norm  \\
\scalebox{0.8}{$\trace\plist{\tr{\pmat} \HmatDN(n,2) \pmat}$}
\end{tabular} 
& 
\begin{tabular}{@{}c@{}}
\scalebox{0.45}{\input{drawings/tikz_second_order_continuous_derivative_n2}} 
\end{tabular}
& 
\begin{tabular}{@{}c@{}}
\scalebox{0.45}{\input{drawings/tikz_second_order_continuous_derivative_n3}} 
\end{tabular}
&
\begin{tabular}{@{}c@{}}
\scalebox{0.45}{\input{drawings/tikz_second_order_continuous_derivative_n4}}
\end{tabular}
\\
\begin{tabular}{@{}c@{}}
Second-Order \\ Difference Norm \\
\scalebox{0.8}{$\sum\limits_{i=1}^{n-1} \norm{\bpoint_{i+1} - 2 \bpoint_i + \bpoint_{i-1}}^2$}\\
\scalebox{0.8}{$\trace\plist{\tr{\pmat} \HmatDDN(n,2) \pmat}$}
\end{tabular} 
& 
\begin{tabular}{@{}c@{}}
\scalebox{0.45}{\input{drawings/tikz_second_order_discrete_difference_n2}} 
\end{tabular}
& 
\begin{tabular}{@{}c@{}}
\scalebox{0.45}{\input{drawings/tikz_second_order_discrete_difference_n3}} 
\end{tabular}
&
\begin{tabular}{@{}c@{}}
\scalebox{0.45}{\input{drawings/tikz_second_order_discrete_difference_n4}}
\end{tabular}
\\
\begin{tabular}{@{}c@{}}
First-Order \\
Difference Variance \\
\scalebox{0.8}{$\sum\limits_{i=1}^{n} n \norm{\bpoint_{i} - \bpoint_{i-1}}^2 - \norm{\bpoint_n - \bpoint_0}^2$}\\
\scalebox{0.8}{$\trace\plist{\tr{\pmat} \HmatDDV(n,1) \pmat}$}
\end{tabular} 
& 
\begin{tabular}{@{}c@{}}
\scalebox{0.45}{\input{drawings/tikz_control_point_jensen_gap_n2}} 
\end{tabular}
& 
\begin{tabular}{@{}c@{}}
\scalebox{0.45}{\input{drawings/tikz_control_point_jensen_gap_n3}} 
\end{tabular}
&
\begin{tabular}{@{}c@{}}
\scalebox{0.45}{\input{drawings/tikz_control_point_jensen_gap_n4}}
\end{tabular}
\\[2mm]
\hline
\hline
\end{tabular}
\vspace{-3mm}
\end{table}

\addtocounter{footnote}{1}
\footnotetext{\label{fn.InteractionGraphWeights}To avoid decimals and enhance readability, we normalize the minimum magnitude of nonzero interaction weights to one.}

\begin{figure*}
\centering
\begin{tabular}{@{}c@{\hspace{0.03\columnwidth}}c@{\hspace{0.005\columnwidth}}c@{\hspace{0.005\columnwidth}}c@{}}
\includegraphics[width = 0.53\textwidth]{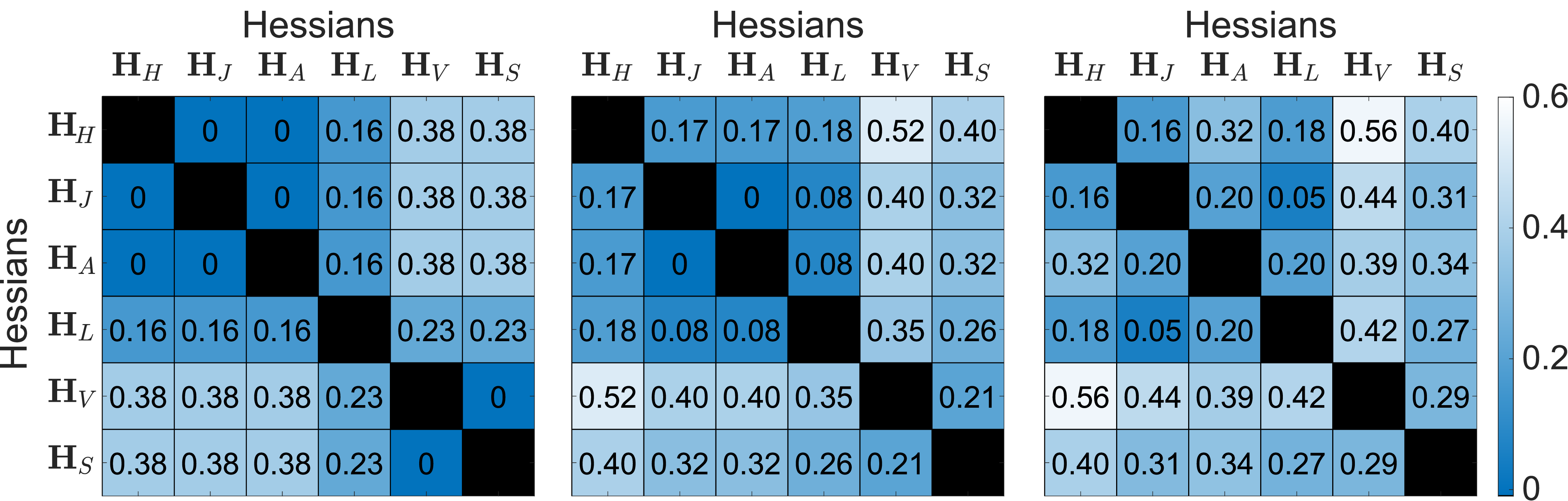}  &
\includegraphics[width = 0.145\textwidth]{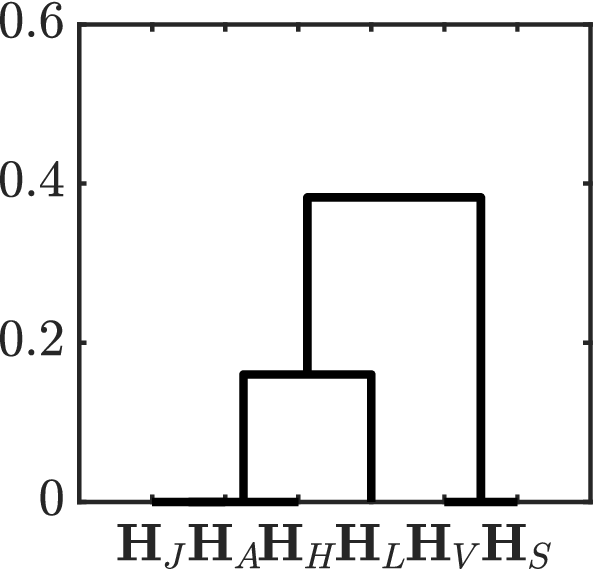}  &
\includegraphics[width = 0.145\textwidth]{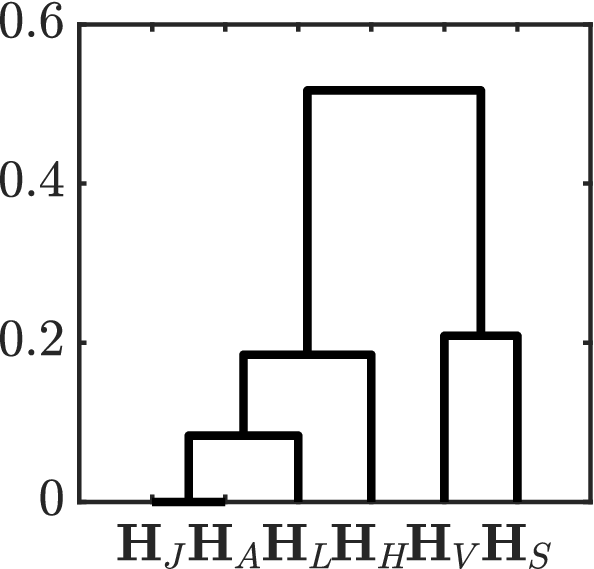}   & 
\includegraphics[width = 0.145\textwidth]{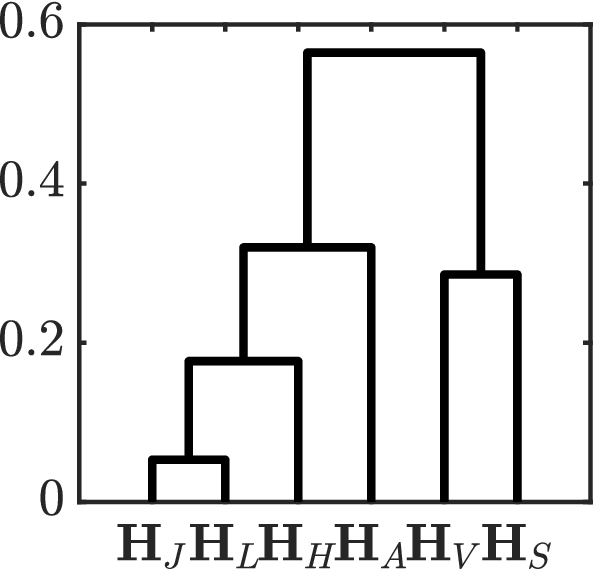}  
\\[-1mm]
\hspace{1mm} \footnotesize{(a)} \hspace{23mm} \footnotesize{(b)} \hspace{23mm} \footnotesize{(c)} &\footnotesize{(d)} & \hspace{2mm} \footnotesize{(e)} &  \hspace{2mm} \footnotesize{(f)}
\end{tabular}
\vspace{-1mm}
\caption{The pairwise distances of the normalized Hessians\reffn{fn.NormalizedHessianDistance} of physical, geometric, and statistical B\'ezier optimization objectives for (a) $n=2$, (b) $n=3$, and (c) $n=4$, and their respective complete-linkage clustering dendrograms in (d), (e), and (f). 
Here, $\Hmat_{V}(n):= \HmatDN(n,1)$ and $\Hmat_{A}(n):= \HmatDN(n,2)$ denote the physical B\'ezier velocity and acceleration norms; $\Hmat_{L}(n):= \HmatDDN(n,1)$ and $\Hmat_{H}(n):=\HmatDDN(n,2)$ denote the geometric B\'ezier length and homogeneity objectives defined as the first-order and the second-order finite-difference norms of the control points; and $\Hmat_{S}(n):=\HmatDDV(n,0)$ and $\Hmat_{J}(n):= \HmatDDV(n,1)$ denote the control-point variance and Jensen gap objectives defined as the zeroth-order and first-order finite-difference variances.}
\label{fig.hessian_similarity_score}
\vspace{-1mm}
\end{figure*}

\subsection{Statistical B\'ezier Optimization Objectives}

Minimizing $k^{\text{th}}$-order B\'ezier derivative (or difference) norm indirectly aims to minimize variation in the $(k-1)^{\text{th}}$-order B\'ezier derivative (or difference). Accordingly, inspired by this observation, we propose the differential variance of  B\'ezier curves as a statistical optimization objective.

\begin{definition}
(\emph{B\'ezier Differential Variance})
The differential variance of a B\'ezier curve can be measured using the variance of its $k^{\text{th}}$-order derivative  or the variance of the $k^{\text{th}}$-order finite difference of its control points as  
\begin{align}
\bvar\plist{\bcurve^{(k)}_{\pmat}} &= \tfrac{n!^2}{(n-k)!^2}\bvar\plist{\bcurve_{\Dmat(n,k)\pmat}} \\
&= \tfrac{n!^2}{(n-k)!^2} \trace\plist{\tr{\pmat} \HmatDV(n,k) \pmat} 
\\
\bvar\plist{\Dmat(n,k) \pmat} & = \tfrac{1}{n - k + 1} \trace\plist{\tr{\pmat} \HmatDDV(n,k) \pmat}
\end{align}
where the Hessians of the variances of B\'ezier derivatives and  B\'ezier control-point differences are, respectively,  given  in terms of the difference matrix $\Dmat(n,k)$ in \refeq{eq.HighOrderFiniteDifference} and the mean-shift matrix $\Smat(n)$ in \refeq{eq.MeanShiftMatrix} by
\begin{align}
\HmatDV(n,k) &:= \tr{\Dmat(n,k)\!} \Smat(n\!-\!k) \HmatN(n\!-\!k) \Smat(n\!-\!k) \Dmat(n,k) \nonumber 
\\
\HmatDDV(n,k)&:= \tr{\Dmat(n,k)\!} \Smat(n\!-\!k) \Dmat(n,k).
\end{align}
\end{definition}

\noindent In \reftab{tab.GraphTheoreticOptimizatinoObjectives}, we present the interaction graphs of B\'ezier control points for the variance of the zeroth-order and the first-order finite difference of B\'ezier control points using the Hessian of B\'ezier difference variance as the graph Laplacian. 
As seen in \reftab{tab.GraphTheoreticOptimizatinoObjectives}, both the variance and norm of the finite difference of B\'ezier control points offer a simpler, reduced interaction relation between the control points compared to B\'ezier derivative norm and variance. 
We also observe that the interaction graph of the second-order derivative norm and the first-order difference variance of B\'ezier curves are the same for the special cases of $n=2$ and $n=3$. 
Hence, we find it useful to present below the explicit form of the first-order difference variance of B\'ezier curves to gain a better intuitive understanding of this optimization objective.      

The first-order difference variance of B\'ezier control points measures Jensen gap of a B\'ezier curve that defines an upper bound on the gap between B\'ezier polygon length and the Euclidean distance between the first and last control point~as\reffn{fn.ControlPointVariance}
\begin{align*}
\bvar(\Dmat(n,1)\pmat) &= \tfrac{1}{n}\trace \plist{\tr{\pmat} \tr{\Dmat(n,1)} \Smat(n-1) \Dmat(n,1) \pmat} \\
& = \tfrac{1}{n} \sum_{i=1}^{n} \sum_{j=1}^{n} \tfrac{1}{2}\norm{\ppoint_{i} - \ppoint_{i-1} - \ppoint_{j} + \ppoint_{j-1}}^2 \\
&= \sum_{i=1}^{n} n \norm{\ppoint_{i} - \ppoint_{i-1}}^2 - \norm{\ppoint_{n} - \ppoint_{0}}^2 
\\
& \geq \plist{\sum_{i=1}^{n} \norm{\ppoint_i - \ppoint_{i-1}}}^{\!\!2} \! -  \norm{\ppoint_{n} - \ppoint_{0}}^2 \geq 0.
\end{align*} 

Finally, to demonstrate the similarity and hierarchical relation among different physical, geometric, and statistical B\'ezier optimization objectives, we present in \reffig{fig.hessian_similarity_score} the pairwise distance matrix of the normalized Hessians of B\'ezier differential norms and variances, as well as the associated clustering dendrograms, for $n= 2, 3, 4$. 
As expected, the second-order (respectively, first-order) differential norms and the first-order (respectively, zeroth-order) differential variances are strongly related to each other since minimizing the acceleration norm is similar to minimizing velocity variation.      

\addtocounter{footnote}{1}
\footnotetext{\label{fn.NormalizedHessianDistance}The pairwise normalized distance of two Hessian matrices $\Hmat_{A}$ and $\Hmat_{B}$ is calculated as $\frac{1}{2}\left\|\frac{\Hmat_{A}}{\norm{\Hmat_{A}}} - \frac{\Hmat_{B}}{\norm{\Hmat_{B}}} \right\|$ where $\norm{\Hmat}^2 = \trace\plist{\tr{\Hmat} \Hmat}$. }

\section{\!\!Safe Corridors for B\'ezier Curve Optimization\!}
\label{sec.SafeCorridors}

Safe corridors are of significant importance in  defining inequality constraints within B\'ezier curve optimization.
In this section, we present a method for constructing convex polygonal safe corridors around a reference path that heuristically balances its distance to obstacles. 

\subsection{Safe Corridor Construction}
\label{sec.SafeCorridorConstruction}  

For motion planning around obstacles, we consider a known closed set of obstacles, denoted by $\obstspace \subset \R^{\cdim}$, and the associated open obstacle-free space  $\freespace := \R^{\cdim} \setminus \obstspace$.
Modelling and understanding the connectivity and topology of free space can be challenging in motion planning \cite{arslan_PhdThesis2016}. 
Identifying a local convex free space, also referred to as a safe corridor, around a safe point in $\freespace$ allows for the exploitation of the local geometry of the free space in robot motion planning and control \cite{arslan_koditschek_IJRR2019, arslan_pacelli_koditschek_IROS2017}. 
Accordingly, we consider convex polygonal safe corridors that enable describing local safety requirements as linear inequality constraints in B\'ezier curve optimization.

\begin{definition}
\emph{(Safe Corridor)} A safe corridor $\safecor(\cpoint, \Xmat)$, parametrized by a center point $\cpoint \in \R^{\cdim}$ and a finite set of critical boundary points $\Xmat = \plist{\xpoint_1, \ldots, \xpoint_m} \in \R^{m \times d}$ is a convex polygonal local free space in $\freespace$ that is defined using the separating hyperplanes passing through the critical boundary points $\Xmat$ and facing towards the center $\cpoint$ as 

\noindent 
\begin{align*}
\safecor(\cpoint, \Xmat) &\!:=\! \clist{\vect{x} \in \R^{\cdim} \Big |  \tr{(\vect{x} \!-\! \xpoint_i)\!}(\cpoint\! -\! \xpoint_i) \geq 0 \quad \forall i = 1, \ldots, m} \nonumber
\\
& \hspace{-3mm}=  \clist{\vect{x} \! \in \R^{\cdim} \big | \tr{(\xpoint_i \!-\! \cpoint)\!} \vect{x} \leq \tr{(\xpoint_i \!-\! \cpoint)\!} \xpoint_i \quad \forall i = 1, \ldots, m} 
\end{align*}
such that the corridor interior, denoted by $\mathring{\safecor}(\cpoint, \Xmat)$, does not intersect the obstacle set, i.e., \mbox{$\mathring{\safecor}(\cpoint, \Xmat) \cap \obstspace = \varnothing$}.
\end{definition}
%
\noindent Hence, a safe corridor can be alternatively  expressed as
\begin{align}
\safecor(\cpoint, \Xmat) & = \clist{\vect{x} \in \R^{\cdim} \big |  \mat{A}_{\safecor(\cpoint, \Xmat)} \vect{x} \leq \vect{b}_{\safecor(\cpoint, \Xmat)}}
\end{align}
where the linear inequality parameters of the safe corridor are determined by the corridor center $\cpoint$ and the critical boundary points $\Xmat = \plist{\xpoint_1, \ldots, \xpoint_m}$ as
\begin{align}\label{eq.SafeCorridorConstraints}
\!\!\mat{A}_{\safecor(\cpoint, \Xmat)}\! :=\! \begin{bmatrix}
\tr{(\xpoint_1 \! -\! \cpoint)} 
\\
\vdots 
\\
\tr{(\xpoint_m \!-\! \cpoint)}
\end{bmatrix}, 
\,
\vect{b}_{\safecor(\cpoint, \Xmat)} \! := \!\begin{bmatrix}
\tr{(\xpoint_1 \!-\! \cpoint)} \xpoint_1
\\
\vdots
\\
\tr{(\xpoint_m \!-\! \cpoint)} \xpoint_n
\end{bmatrix}. \!\! \!
\end{align}

One can construct a maximal-volume local free space around a given corridor centroid $\cpoint$ by incrementally identifying and eliminating the closest obstacle points using separating hyperplanes, as illustrated in \reffig{fig.single_polygon}. 
Given any nonempty obstacle set $\obstspace$ and any desired collision-free corridor center $\cpoint \in \freespace$ with $\min_{\xpoint_{\obstspace} \in \obstspace} \norm{\cpoint - \xpoint_{\obstspace}} > 0$, the critical boundary points $\Xmat$ of the maximal-volume safe corridor that is centered at $\cpoint$ can be incrementally determined as follows:

\begin{enumerate}[i)]
\item (Initialize) Start with the empty boundary point set, i.e.,
\begin{align*}
\Xmat \leftarrow \varnothing
\end{align*}

\item \label{step.CollisionPointUpdate}(Update) Find the closest point of the obstacle set in the corridor interior to the corridor center, and then add it to the critical boundary point set,
\begin{align*}
\xpoint^*  &= \argmin_{\xpoint_{\obstspace} \in \obstspace \cap \mathring{\safecor}(\cpoint, \Xmat)} \norm{\xpoint_{\obstspace} - \cpoint} 
\\
\Xmat &\leftarrow \Xmat \cup \clist{\xpoint^*}
\end{align*}

\item (Terminate) If $\mathring{\safecor}(\cpoint, \Xmat) \cap \obstspace  \neq \varnothing$, then go to step \ref{step.CollisionPointUpdate}. Otherwise, terminate with the boundary point set $\Xmat$.
\end{enumerate}  

\noindent 
Note that this way of identifying safe corridor boundary points defines a well-defined functional relation for safe corridor construction, denoted by $f_{\safecor}(\ppoint, \obstspace)$. Moreover, the safe corridor construction function $f_{\safecor}(\ppoint, \obstspace)$  always returns critical corridor  boundary points that define nonredundant linear inequality constraints.
It is also pertinent to mention that if the obstacle set consists of a finite collection of discrete points or convex sets, then the critical boundary point set is finite because $\Xmat$ might contain a boundary point corresponding to each discrete obstacle point or convex obstacle set. 
For example, standard point cloud sensors provide information about discrete obstacle points whereas binary occupancy grid maps represent obstacles using square-shaped occupied grid cells. 
To efficiently compute a safe corridor using a finite collection of obstacle points, one can start sorting the obstacle points according to their distances to the corridor center so that the search for the relevant closest obstacle point in the update step can be performed by recursively eliminating obstacle points outside the corridor. 
To construct a safe corridor over a binary occupancy grid map, one can employ the grid structure to perform a circular spiral search, effectively determining all relevant critical boundary points without the need to explore the entire map.

\begin{figure}[t]
\centering
\begin{tabular}{@{}c@{\hspace{0.5mm}}c@{\hspace{0.5mm}}c@{}}
\includegraphics[width = 0.33\columnwidth]{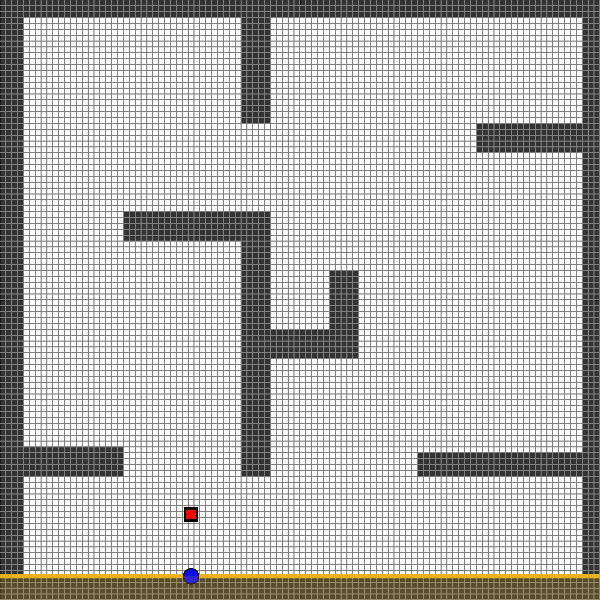}   
& 
\includegraphics[width = 0.33\columnwidth]{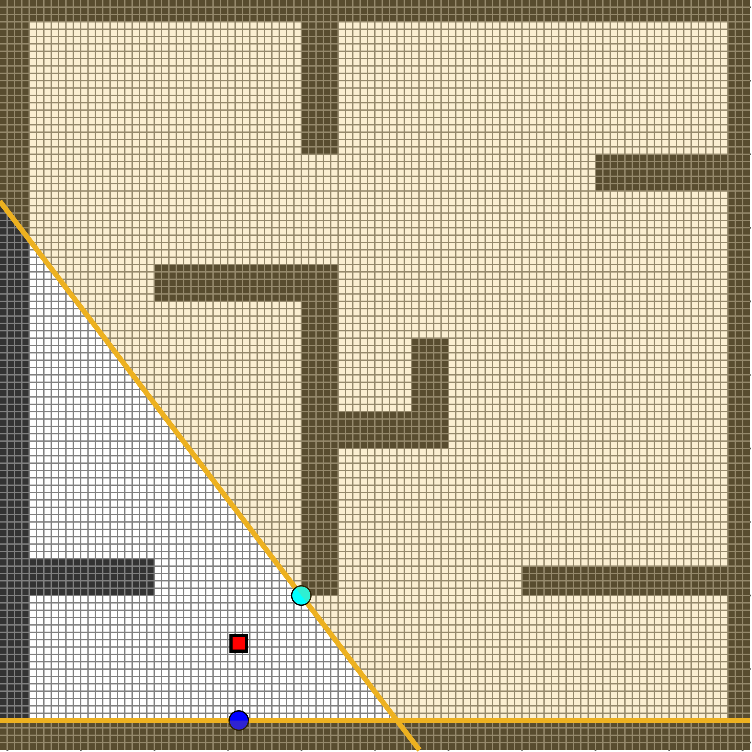}  
&
\includegraphics[width = 0.33\columnwidth]{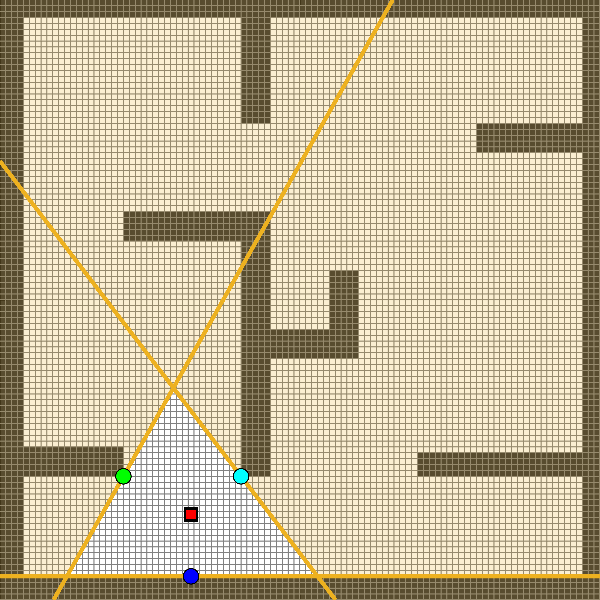}  
\end{tabular}
\vspace{-2mm}
\caption{Construction of a convex polygonal safe corridor around a given corridor center (red square) using separating hyperplanes from obstacles.}
\label{fig.single_polygon}
\vspace{-2mm}
\end{figure}

An important application of safe corridors is identifying the available local free space around a given reference path.
To minimize the number of safe corridors along the reference path, one can begin by placing a safe corridor at the start of the reference path and then incrementally search for the next reference path point to position the subsequent safe corridor center.
Here, the search constraint is that the path segment between the previous and next corridor centers must stay within the previous safe corridor
Accordingly, given a collision-free reference path $\refpath: [0,1] \rightarrow \freespace$ with some positive clearance from obstacles, i.e., \mbox{$\min_{\substack{t \in [0,1]\\ \xpoint_{\obstspace} \in \obstspace}}\norm{\refpath(t) - \xpoint_{\obstspace}} > 0$}, we construct a sequence of adjacent safe corridors with intersections as illustrated in \reffig{fig.complete_corridor} and described below: 
\begin{enumerate}[i)]
\item (Initialize) Construct a safe corridor at the start of the reference path: for $k \leftarrow 1$ 
\begin{align*}
t_k &\leftarrow 0 \\
\Xmat_k &\leftarrow f_{\safecor}(\refpath[t_k], \obstspace)
\end{align*}
\item \label{step.UpdateCorridorCenter} (Update) Find the reference path point that is furthest from the previous corridor center,  ensuring that the associated reference path segment remains within the previous safe corridor, and construct a safe corridor at this next reference path point: for $k  \leftarrow k + 1$
\begin{align*}
t_{k}& =   \argmax_{\substack{t \in [t_{k-1}, 1]\\ r([t_{i-1}, t]) \subseteq \safecor\plist{r(t_{i-1}), \Xmat_{i-1}}}} t\\
\Xmat_k &\leftarrow f_{\safecor}(\refpath[t_{k}], \obstspace)
\end{align*}   

\item (Terminate) If $t_k < 1$, then go to step \ref{step.UpdateCorridorCenter}. Otherwise, terminate with the safe corridor centers $\plist{\refpath(t_1), \ldots, \refpath(t_k)}$ and the critical boundary points $\plist{\Xmat_{1}, \ldots, \Xmat_{k}}$.

\end{enumerate} 
\noindent Note that the number and size of safe corridors constructed around a reference path depend primarily on the reference path's clearance from obstacles and its length.   
A shorter reference path with greater clearance results in fewer and larger safe corridors.

\begin{figure}[t]
\centering
\begin{tabular}{@{}c@{\hspace{0.5mm}}c@{\hspace{0.5mm}}c@{}}
\includegraphics[width = 0.33\columnwidth]{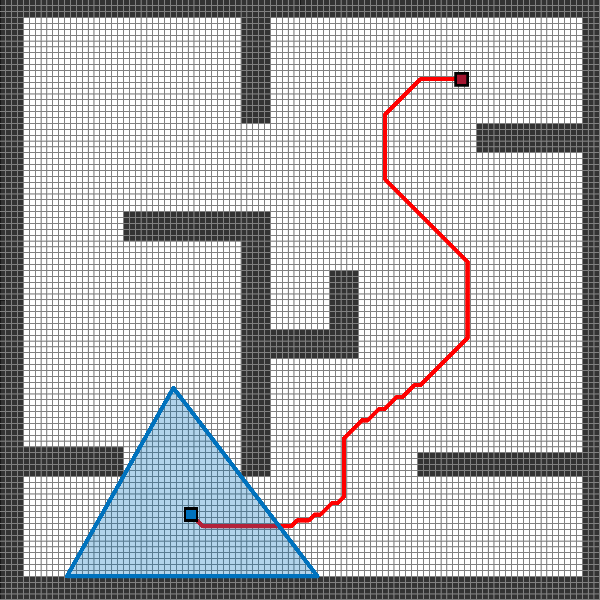}   
& 
\includegraphics[width = 0.33\columnwidth]{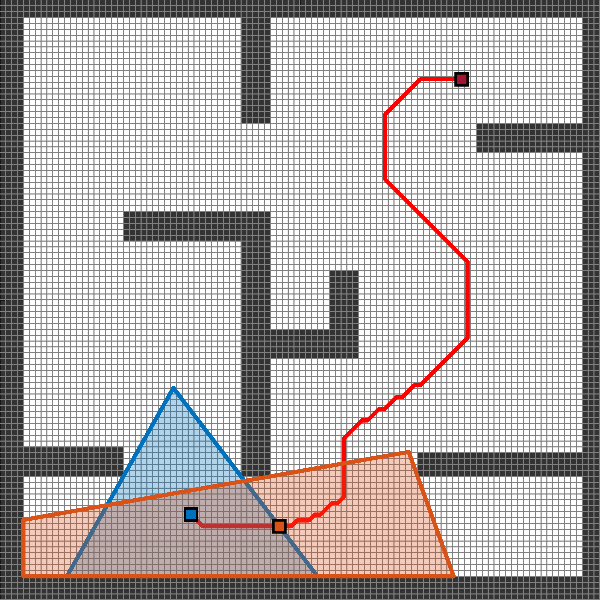}  
&
\includegraphics[width = 0.33\columnwidth]{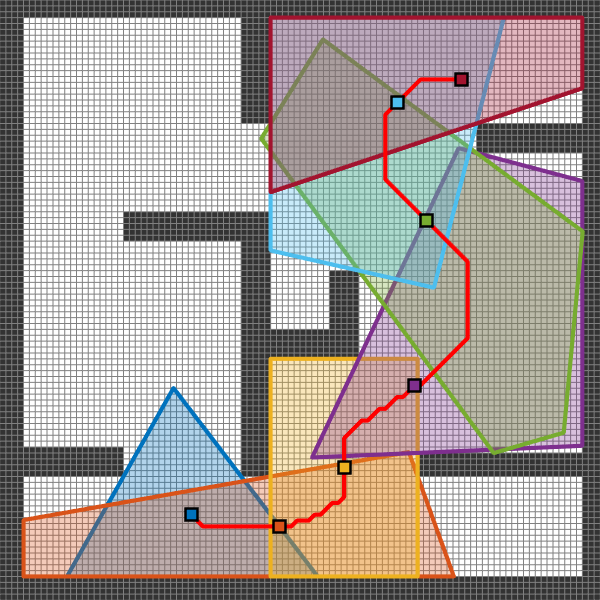} 
\end{tabular}
\vspace{-2mm}
\caption{Construction of safe corridors along a safe reference path (red line) from the start (blue square) to the goal (red square).}
\label{fig.complete_corridor}
\vspace{-2mm}
\end{figure}


\begin{figure}[b]
\vspace{-3mm}
\centering
\includegraphics[width = 0.43\columnwidth]{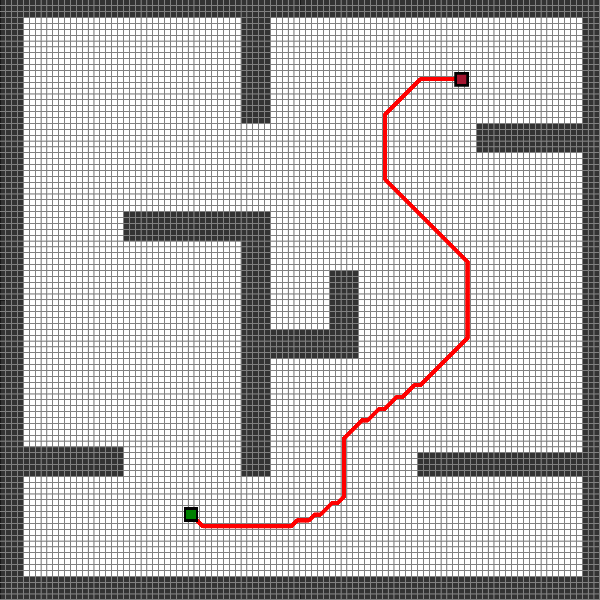}  
\includegraphics[width = 0.50\columnwidth]{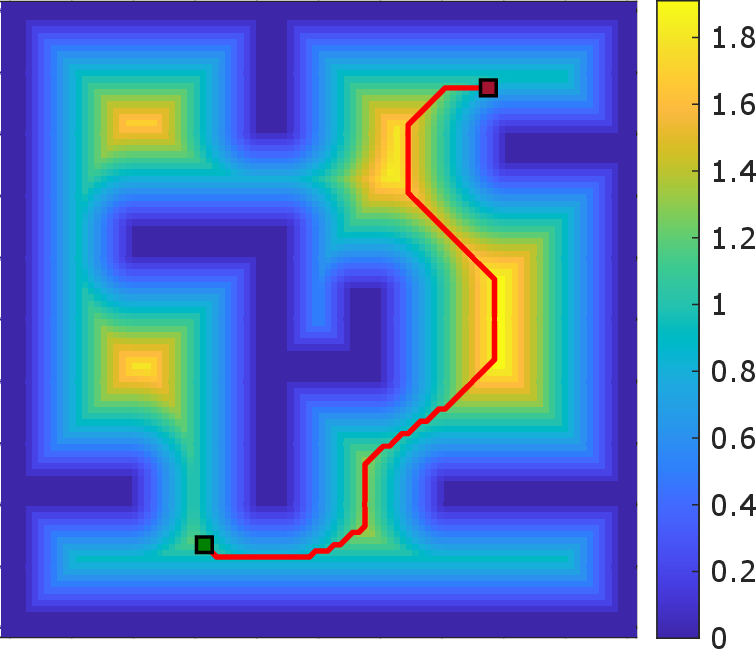}
\vspace{-2mm} 
\caption{A maximal-clearance minimal-length reference path over a binary grid map (left) is constructed using graph search with the inverse of the distance field (right) as a cost map to balance clearance from obstacles.}
\label{fig:ref_path}
\end{figure}

\subsection{Reference Path Planning for Safe Corridor Construction}
\label{sec.ReferencePathPlanning}

To minimize the number of safe corridors and maximize their volume, we adopt a heuristic reference path construction strategy that maximizes clearance from obstacles while simultaneously minimizing the length of the reference path.
In particular, to find a reference path with maximal clearance and minimal length over a binary occupancy grid map, we utilize the inverse distance field, representing the inverse of the distance to collision, as the visiting cost for each map grid cell, i.e., 
{\small
\begin{align*}
\mathrm{cost}(\mathrm{grid}) =  \frac{1}{\mathrm{dist2coll}(\mathrm{grid})}. 
\end{align*}
}%
We then determine the minimum-cost path joining a given pair of start and goal positions,  using a graph search (e.g., A* or Dijkstra) algorithm over a binary occupancy map (assuming 8-neighbor connectivity) where the edge transition cost between two adjacent grid cells is set as the maximum of their individual visiting costs, i.e.,
{\small
\begin{align*}
\mathrm{travel\_cost}(\mathrm{grid}_i, \mathrm{grid}_j)& =  \max\plist{\mathrm{cost}(\mathrm{grid}_i), \mathrm{cost}(\mathrm{grid}_j)}\\
& \hspace{-13mm} = \frac{1}{\min\plist{\mathrm{dist2coll}(\mathrm{grid}_i), \mathrm{dist2coll}(\mathrm{grid}_j)}} .  
\end{align*}
}%
In \reffig{fig:ref_path}, we illustrate an example of a maximal-clearance minimal-length reference path over a binary occupancy grid map that connects a given start and goal positions while ensuring a balanced distance from nearby obstacles.

\section{Numerical Simulations}
\label{sec.NumericalSimulations}
 
In this section, we present numerical examples to demonstrate graph-theoretic B\'ezier curve optimization over safe corridors on a binary occupancy map using different physical, geometric and statistical B\'ezier optimization objectives.
For a given start and goal positions, denoted by $\ppoint_{\mathrm{start}}$ and $\ppoint_{\mathrm{goal}}$, we first find a maximal-clearance minimal-length reference path joining $\ppoint_{\mathrm{start}}$ and $\ppoint_{\mathrm{goal}}$ as described in \refsec{sec.ReferencePathPlanning}. 
Then, we automatically deploy a collection of safe corridors, denoted by $\safecor_1, \ldots, \safecor_m$, along the reference path as described in \refsec{sec.SafeCorridorConstruction}.
   
We assume that each safe corridor is associated with a B\'ezier curve for which it defines safety constraints.
Hence, our optimization variables are the control points of a sequence of connected B\'ezier curves $\plist{\bcurve_{\pmat_1}, \ldots, \bcurve_{\pmat_m}}$  that are constrained over the associated safety corridors $\safecor_1, \ldots, \safecor_m$ and smoothly join the start position  $\ppoint_{\mathrm{start}}$ and the goal position $\ppoint_{\mathrm{goal}}$.
We assume that there is a certain desired level $C$ of continuity between adjacent B\'ezier curves. 
Accordingly, as in \reftab{tab.BezierCurveOptimization}, we consider an explicit matrix formulation of the generic graph-theoretic B\'ezier curve optimization over safe corridors as 
\begin{align*}
\begin{array}{rl}
\textrm{minimize} & \sum_{i=1}^{m} \trace \plist{\tr{\pmat}_{i} \mat{L}_i \pmat_i} \\[1mm]
\textrm{subject to} & \tr{\pmat_1} 
\scalebox{0.6}{$\begin{bmatrix} 1 \\ 0 \\ \vdots \\ 0\end{bmatrix}$} 
= \vect{p}_{\mathrm{start}}$, $\tr{\pmat_m} 
\scalebox{0.6}{$\begin{bmatrix}0 \\ \vdots\\ 0\\ 1 \end{bmatrix}$} 
= \vect{p}_{\mathrm{goal}} 
\vspace{1mm}
\\
& \hspace{-7mm} \tr{\pmat_i\!}\tr{\Dmat(n,c)\!} \scalebox{0.6}{$\begin{bmatrix}0 \\ \vdots\\ 0\\ 1 \end{bmatrix}$} 
 \!=\! \tr{\pmat_{i+1}\!\!}\tr{\Dmat(n,c)\!} \scalebox{0.6}{$\begin{bmatrix}1 \\ 0 \\ \vdots\\ 0 \end{bmatrix}$} 
 \quad \substack{\forall i = 1, \dots, m-1\\ \forall c = 0, \ldots, C \hspace{4mm}}
\vspace{1mm}
\\
& \mat{A}_{\safecor_i}\tr{\pmat_i} \leq \vect{b}_{\safecor_i} \tr{\mat{1}} \quad \scalebox{0.75}{$\forall i = 1, \dots, m-1$}
\end{array}
\end{align*}
where $\mat{L}_i$ is a positive semidefinite Laplacian matrix describing a desired interaction relation for the control points of B\'ezier curve $\bcurve_{\pmat_i}$, $\Dmat(n,k)$ is the finite-difference matrix in \refeq{eq.HighOrderFiniteDifference}, $\mat{A}_{\safecor_i}$ and $\vect{b}_{\safecor_i}$ are the linear inequality parameters of the safe corridor $\safecor_i$ defined as in \refeq{eq.SafeCorridorConstraints}.     

In \reffig{fig:optimized_paths_geo}, we consider various B\'ezier differential norms and variances as an optimization objective and illustrate the resulting optimal third-order B\'ezier curves (with $C^1$ endpoint continuity) over convex polygonal safe corridors constructed along a reference path that balances distance to obstacles while connecting the start and goal positions. 
Here, the major use of the reference path is to determine safe corridor constraints for B\'ezier curve optimization.
Therefore, a reference path acts as a topological description of a navigation corridor that is used to guide B\'ezier curve optimization to find a safe and smooth optimal polynomial path joining the start and goal position.     
As expected, the low-order differential B\'ezier optimization objectives yield shorter paths with higher curvature while higher-order differential  B\'ezier optimization objectives result in longer but smoother paths with lower curvature. 
The similarity of the velocity norm and the control-point variance is due to the strong similarity of their interaction graphs and Hessians/Laplacians in \reftab{tab.GraphTheoreticOptimizatinoObjectives} and \reffig{fig.hessian_similarity_score}, which also explains the relation between the acceleration norm and the velocity variance.
Hence, based on these interaction graphs and the special cases discussed in \refsec{sec.BezierOptimizationObjective}, one can conclude that B\'ezier difference norms and variances offer a more intuitive geometric interpretation compared to B\'ezier derivative norms and variances, despite providing comparable optimal curve profiles. 
Lastly, we observe that the conservatism of  B\'ezier curve convexity, when over-approximating safety constraints, becomes less severe, and even negligible, with increasing B\'ezier curve degree. 
We also see that having an extra high degree of B'ezier curve, higher than the continuity requirements of the optimization problem, does not change the resulting path quality which is due to the degree elevation property of B\'ezier curves \cite{arslan_tiemessen_TRO2022}.

\begin{figure*}[t]
\centering
\begin{tabular}{@{}c@{\hspace{0.005\columnwidth}}c@{\hspace{0.005\columnwidth}}c@{\hspace{0.005\columnwidth}}c@{\hspace{0.005\columnwidth}}c }
\includegraphics[width = 0.197\textwidth]{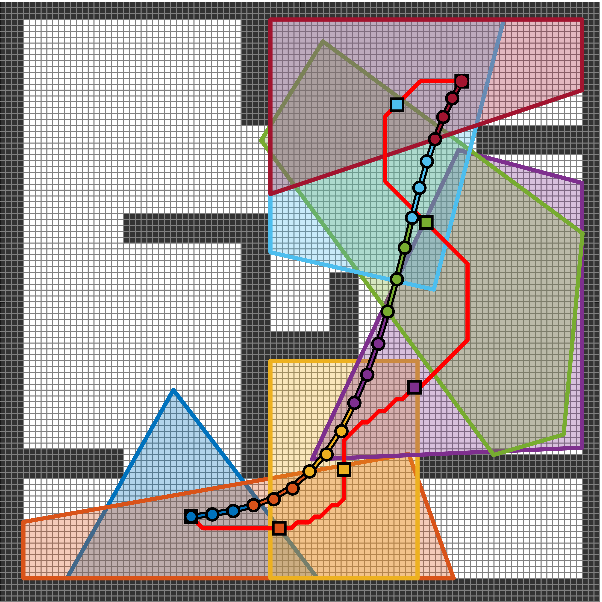}  &
\includegraphics[width = 0.197\textwidth]{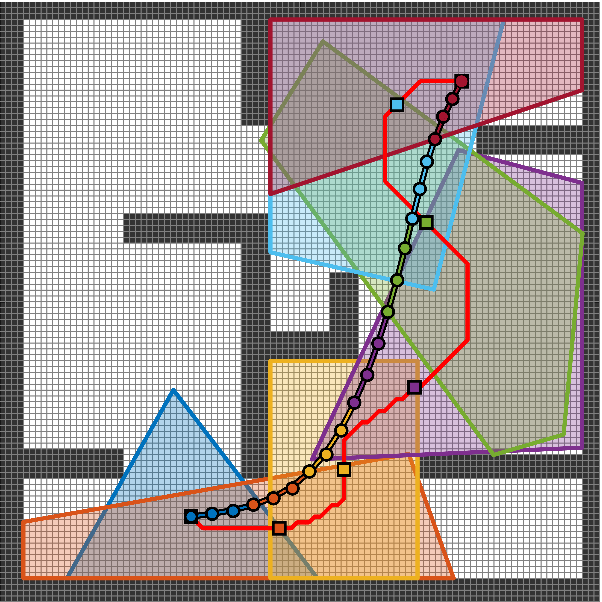}  &
\includegraphics[width = 0.197\textwidth]{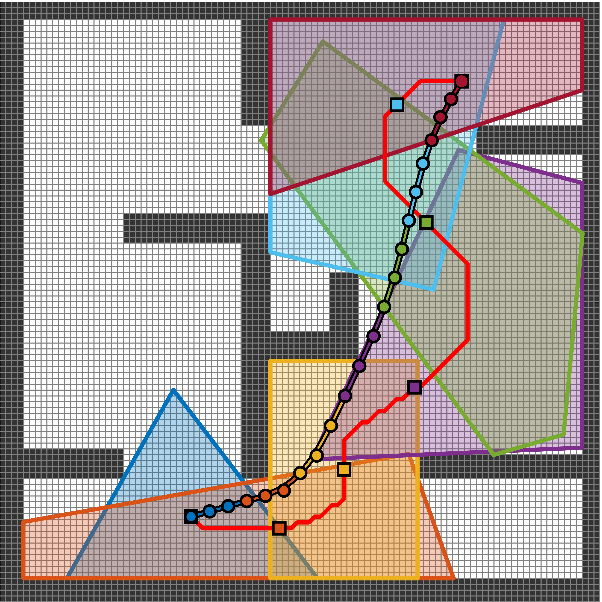} & 
\includegraphics[width = 0.197\textwidth]{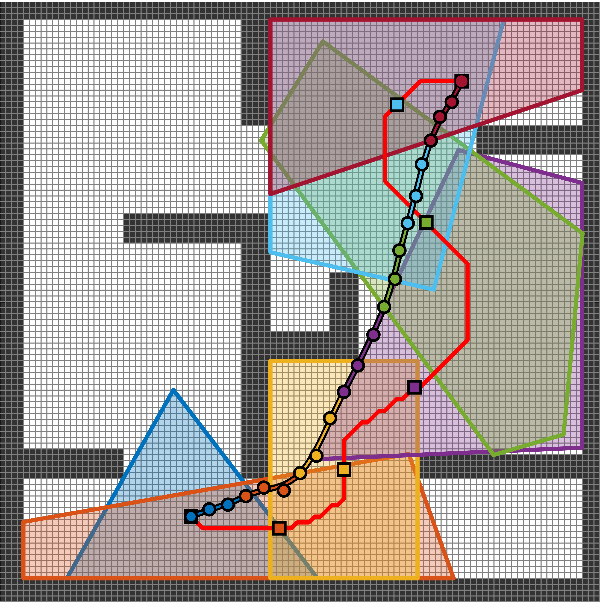} & 
\includegraphics[width = 0.197\textwidth]{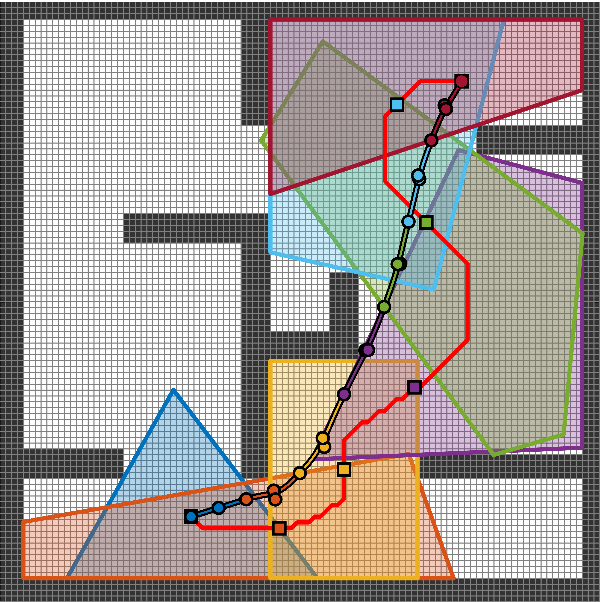}
\\
\includegraphics[width = 0.197\textwidth]{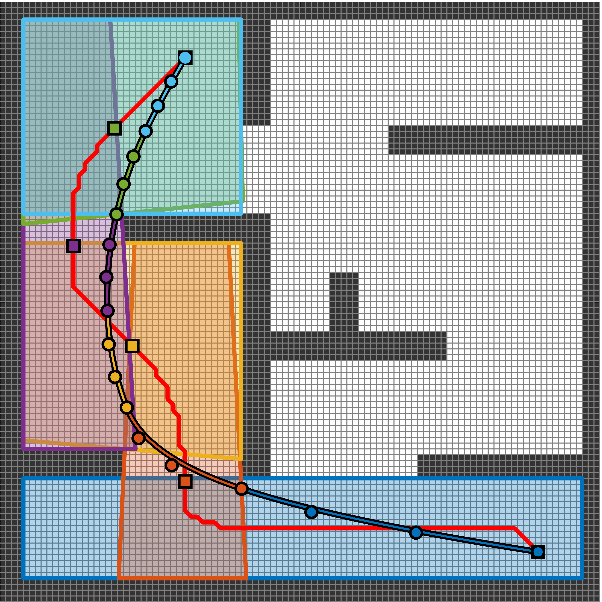}  &
\includegraphics[width = 0.197\textwidth]{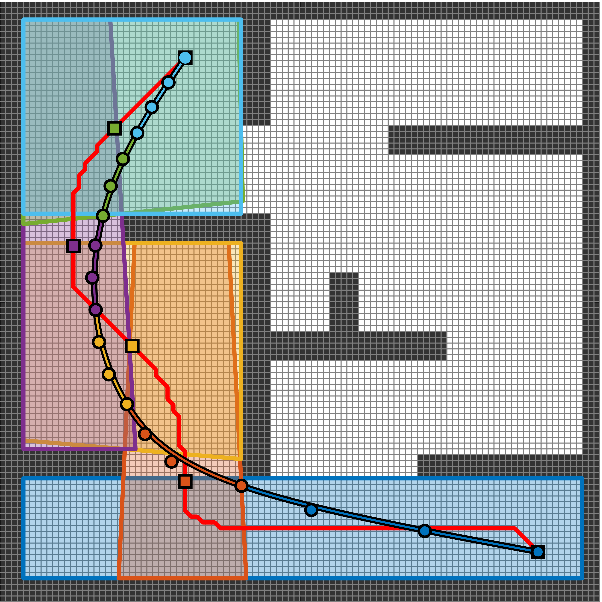}  &
\includegraphics[width = 0.197\textwidth]{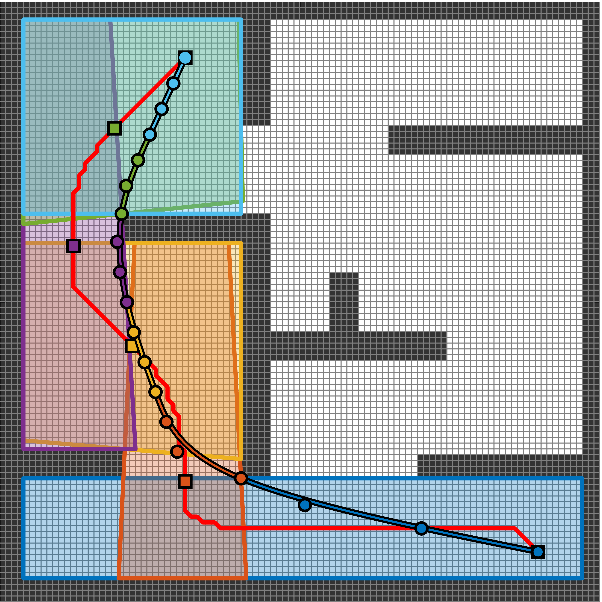} & 
\includegraphics[width = 0.197\textwidth]{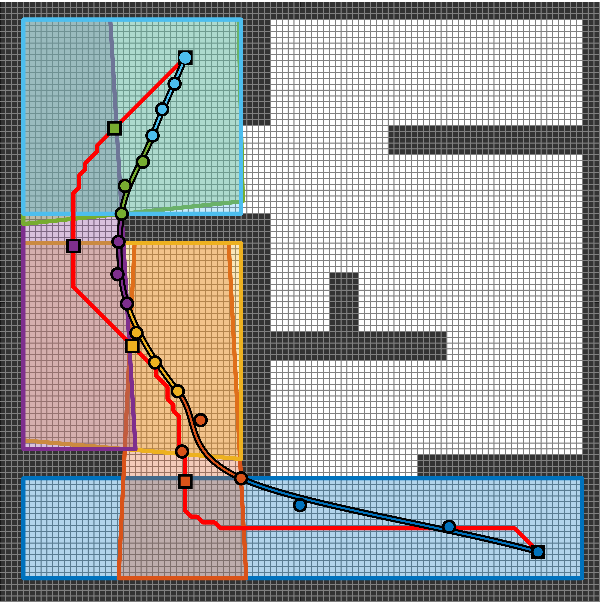} & 
\includegraphics[width = 0.197\textwidth]{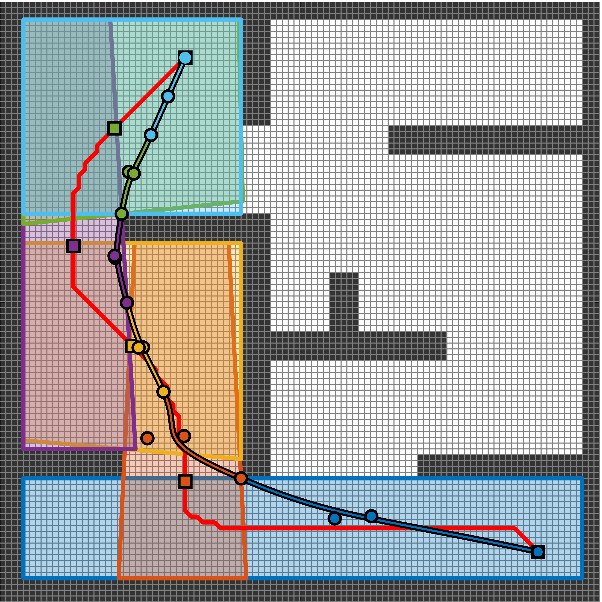}
\\
\includegraphics[width = 0.197\textwidth]{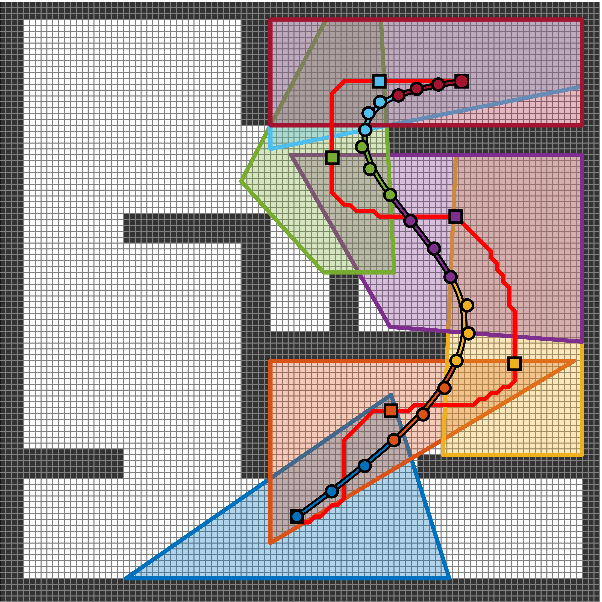}  &
\includegraphics[width = 0.197\textwidth]{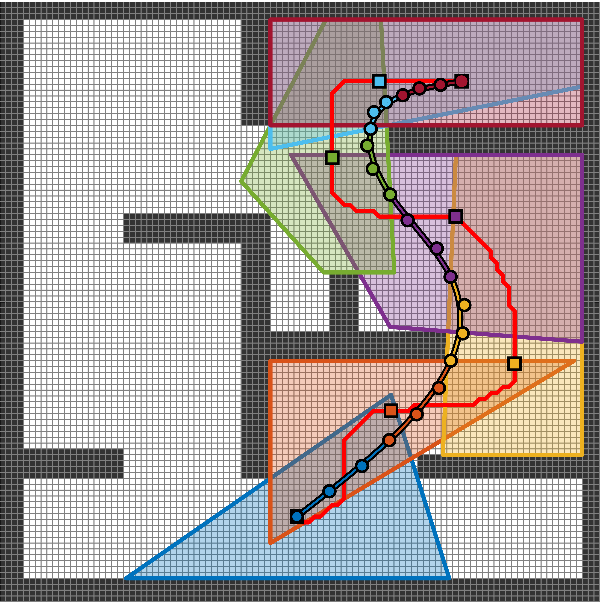}  &
\includegraphics[width = 0.197\textwidth]{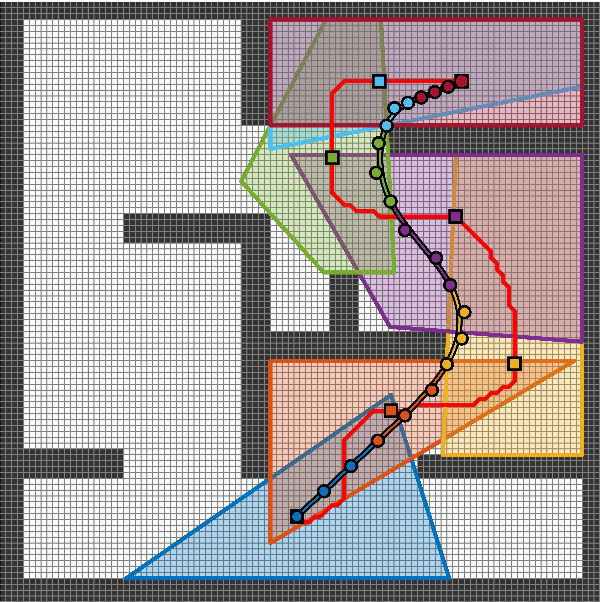} & 
\includegraphics[width = 0.197\textwidth]{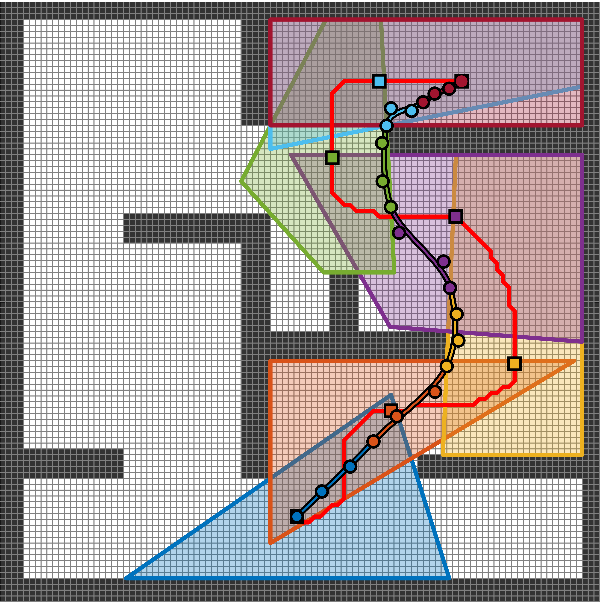} & 
\includegraphics[width = 0.197\textwidth]{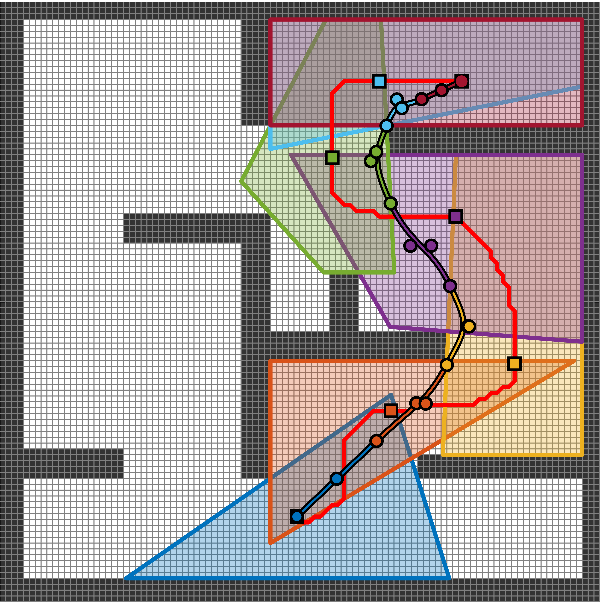}
\\[-1mm]
\footnotesize{(a)} & \footnotesize{(b)} & \footnotesize{(c)} & \footnotesize{(d)} & \footnotesize{(e)}
\end{tabular}
\vspace{-4mm}
\caption{Third-order B\'ezier curve optimization with  $C^{1}$ endpoint continuity over convex polygonal safe corridors constructed along a maximal-clearance minimal-length reference path (red): (a) the second-order difference norm, (b) the  second-order derivative norm and the first-order difference variance, (c) the first-order difference norm, (d) the first-order derivative norm, (e) the zeroth-order difference variance.}
\label{fig:optimized_paths_geo}
\end{figure*}

\section{Conclusions}
\label{sec.Conclusions}


In this paper, we introduce a new graph-theoretic approach for designing B\'ezier curve optimization objectives that measures the alignment of B\'ezier control points using a quadratic consensus distance derived from the Laplacian of an interaction graph of the control points. 
We present examples of physical, geometric and statistical consensus distances explicitly constructed using the norms and variances of B\'ezier curve derivatives and B\'ezier control-point differences. 
We demonstrate the use of these quadratic optimization objectives for finding safe and smooth paths over safe corridors.  
We present an explicit matrix formulation for such constrained quadratic optimization of B\'ezier curves.  
We observe that the norm and variance of B\'ezier continuous derivatives and finite discrete differences of the same order  generate similar path profiles.
Consequently, owing to the simpler coupling between control points, we conclude that the norm and variance of the finite differences of B\'ezier control points lead to more intuitive geometric optimization objectives compared to B\'ezier derivatives.
We believe that the unifying graph-theoretic interpretation of B\'ezier curve optimization opens up new research challenges, including the design of novel interaction graphs for B\'ezier control points to create new optimization objectives that go beyond B\'ezier differential norms and variances, for example, for approximately assessing (maximum) curve curvature for curvature-constrained smooth motion planning. 
Another promising research direction is to perform sensor-based B\'ezier curve optimization for local planning in dynamic environments~\cite{arslan_koditschek_IJRR2019}. 





\bibliographystyle{IEEEtran}
\bibliography{references}


\appendices 

\section{Proofs}
\label{app.Proofs}

\subsection{Proof of \reflem{lem.BezierInnerProduct}}
\label{app.BezierInnerProduct}

\begin{proof}
The result can be verified using the linearity and cyclic property of the trace operator as
\begin{align*}
\int_{0}^{1} \tr{\bcurve_{\pmat}(t)} \bcurve_{\qmat}(t) \diff t &= \int_{0}^{1} \tr{\bbasis_{n}(t)}\pmat  \tr{\qmat} \bbasis_{m}(t) \diff t
\\
&= \int_{0}^{1} \trace \plist{\tr{\qmat} \bbasis_{m}(t) \tr{\bbasis_{n}(t)} \pmat} \diff t 
\\ 
&= \int_{0}^{1} \trace \plist{\tr{\pmat} \bbasis_{n}(t) \tr{\bbasis_{m}(t)} \qmat} \diff t 
\\
&= \trace\plist{\tr{\pmat} \int_{0}^{1} \bbasis_{n}(t) \tr{\bbasis_{m}(t)} \diff t \qmat}
\\
& = \trace\plist{\tr{\pmat} \HmatB(n,m) \qmat}
\end{align*}
where each element of  $\HmatB(n,m) = \int_{0}^{1} \bbasis_{n}(t) \tr{\bbasis_{m}(t)} \diff t 
$ can be explicitly calculated  for $i = 0, \ldots n$ and $j=0, \ldots m$ as
\begin{align*}
\blist{\HmatB(n,m)}_{i+1, j+1} &= \int_{0}^1 \bbasis_{i,n}(t) \bbasis_{j,m}(t) \diff t 
\\
&= \frac{\binom{n}{i} \binom{m}{j}}{\binom{m+n}{i+j}}\int_{0}^{1} \bbasis_{i+j, m+n}(t) \diff t 
\\
& = \tfrac{1}{m+n+1} \frac{\binom{n}{i} \binom{m}{j}}{\binom{m+n}{i+j}}
\end{align*}
which is due to the product and integral properties of Bernstein polynomials, i.e., $\bbasis_{i,n}(t) \bbasis_{j,m}(t) = \frac{\binom{n}{i} \binom{m}{j}}{\binom{m+n}{i+j}} \bbasis_{i+j, m+n}(t)$ and $\int_{0}^{1} \bbasis_{i,n}(t) \diff t = \frac{1}{n+1}$ for all $i = 0, \ldots, n$ \cite{farin_CurvesSurfaces2002}.
\end{proof}

\subsection{Proof of \reflem{lem.BezierStatistics}}
\label{app.BezierStatistics}

\begin{proof}
The mean relation simply follows from the constant definite integral property of Bernstein polynomials, i.e., $\int_{0}^{1}\bpoly_{i,n}(t) \diff t  = \frac{1}{n+1}$ \cite{farouki_CAGD2012}, as 
\begin{align*}
\bmean(\bcurve_{\bpmat}) &= \int_{0}^{1} \bcurve_{\bpmat}(t) \diff t = \int_{0}^{1} \sum_{i=0}^{n} \bpoly_{i,n}(t) \bpoint_i \diff t 
\\
& = \sum_{i=0}^{n} \bpoint_i \underbrace{\int_{0}^{1} \bpoly_{i,n}(t)\diff t}_{\frac{1}{n+1}} = \tfrac{1}{n+1}\sum_{i=0}^{n+1} \bpoint_i = \bmean(\bpmat).
\end{align*}
The variances of a B\'ezier curve and its control points can be calculated using the fact that $\bcurve_{\bpmat}(t) - \bmean(\bcurve_{\bpmat}) = \bcurve_{\Smat(n) \bpmat} (t)$  and $\trace\plist{\tr{\pmat} \pmat} = \sum_{i=0}^{n} \norm{\ppoint_i}^2$ as
\begin{align*}
\bvar(\bcurve_{\pmat}) & = \int_{0}^{1} \norm{\bcurve_{\bpmat}(t) - \bmean(\bcurve_{\bpmat})}^2 \diff t = \int_{0}^{1} \norm{\bcurve_{\Smat(n) \bpmat}(t)}^2 \diff t 
\\ 
& = \trace \plist{\tr{\pmat} \Smat(n) \HmatN(n) \Smat(n) \pmat}
\\
\bvar(\pmat) & = \tfrac{1}{n+1}\sum_{i=0}^{n} \norm{\bpoint_i - \bpmean(\bcurve_{\bpmat})}^2 = \tfrac{1}{n+1}\trace\plist{\tr{\pmat} \Smat(n) \pmat}.
\end{align*}
Hence, $\bvar(\bcurve_{\bpmat}) \leq \frac{1}{n+1} \bvar(\bpmat)$ since we have\reffn{fn.BezierNormHessianBound}  $\mat{0} \preceq \HmatN(n) \preceq \frac{1}{n+1} \mat{I}$ and $\Smat(n) \Smat(n) = \Smat(n)$, which completes the proof.
\end{proof}

\section{Higher-Order Bernstein \& B\'ezier Derivatives}

Higher-order derivatives of Bernstein basis polynomials can be calculated explicitly \cite{farin_CurvesSurfaces2002, farouki_CAGD2012}.

\begin{lemma}
\emph{(Higher-Order Bernstein Derivatives)} The $k^{\text{th}}$-order derivative of the $i^{\text{th}}$ Bernstein basis polynomial $\bpoly_{i,n}(t) = \binom{n}{i} t^{i}(1-t)^{n-i}$ of degree $n$ is given by
\begin{align*}
\frac{\diff ^{k}}{\diff t ^k} \bpoly_{i, n}(t) = \frac{n!}{(n-k)!} \sum_{j=0}^{k} \binom{k}{j}(-1)^{k-j} \bpoly_{i-j, n-k}(t).
\end{align*}
\end{lemma}
\begin{proof}
The results can be verified using the general Leibniz rule, $(fg)^{(k)} = \sum_{j=0}^{k} \binom{k}{j} f^{(j)}g^{(k-j)}$, as 
\begin{align*}
\frac{\diff ^{k}}{\diff t^{k}} \bpoly_{i, n}(t) \hspace{-11mm} & \hspace{+11mm}=  \frac{\diff ^{k}}{\diff t^{k}} \tbinom{n}{i} t^{i} (1-t)^{n-i} \\
& = \tbinom{n}{i}\sum_{j=0}^{k} \tbinom{k}{j} \hspace{-2mm}\underbrace{\frac{\diff ^{j}}{\diff t^{j}} t^{i}}_{= \frac{i!}{(i-j)!}t^{i-j}} \hspace{-11mm}\overbrace{\frac{\diff ^{k-j}}{\diff t^{k-j}} (1-t)^{n-i}}^{= \frac{(n-i)!}{(n-k-i+j)!}(1\!-\!t)^{n-k-i+j} (-1)^{k-j}}
\\
& = \tbinom{n}{i}\sum_{j=0}^{k} \tbinom{k}{j} \tfrac{i!}{(i-j)!} t^{i-j} \tfrac{(n-i)!}{(n-k - i+j)!} (1-t)^{n-k-i+j} (-1)^{k - j}
\\
&= \sum_{j=0}^{k} \tbinom{n}{i} \tbinom{k}{j} \tfrac{i!}{(i-j)!} \tfrac{(n-i)!}{(n-k - i+j)!} \frac{1}{\binom{n-k}{i-j}} (-1)^{k - j}\bpoly_{i-j, n-k}(t)
\\
&= \sum_{j=0}^{k} \tbinom{n}{i} \tbinom{k}{j} \frac{n!}{\binom{n}{i}} \frac{\binom{n-k}{i-j}}{(n-k)!} \frac{1}{\binom{n-k}{i-j}} (-1)^{k - j}\bpoly_{i-j, n-k}(t)
\\
& = \tfrac{n!}{(n-k)!} \sum_{j=0}^{k} \tbinom{k}{j} (-1)^{k-j} \bpoly_{i-j, n-k}(t)
\end{align*}
where $\frac{\diff^j}{\diff t ^{j}} t^{i} = \frac{i!}{(i-j)!} t^{i-j}$ and $\frac{\diff^j}{\diff t ^{j}} (1\!-\! t)^{i} = \frac{i! (-1)^{j}}{(i-j)!} (1\!-\!t)^{i-j} $.
\end{proof}

Accordingly, higher-order derivatives of a B\'ezier curve can be analytically determined as
\begin{align*}
\frac{\diff^{k}}{\diff t^{k}} \bcurve_{\bpoint_0, \ldots, \bpoint_n}(t) &= \sum_{i=0}^{n} \bpoint_i \frac{\diff^{k}}{\diff t^{k}} \bpoly_{i,n}(t) 
\\
&= \tfrac{n!}{(n-k)!} \sum_{i=0}^{n} \bpoint_i \sum_{j = 0}^{k} \binom{k}{j}(-1)^{k-j} \bpoly_{i-j, n-k}(t) 
\\
& =  \tfrac{n!}{(n-k)!}\sum_{i=0}^{n-k} \underbrace{\sum_{j = 0}^{k} \binom{k}{j} (-1)^{k-j}\bpoint_{i+j}}_{\qpoint_i} \bpoly_{i, n-k}(t)
\\
& =  \tfrac{n!}{(n-k)!}\sum_{i=0}^{n-k} \qpoint_i \bpoly_{i, n-k}(t) 
\\
&= \tfrac{n!}{(n-k)!} \bcurve_{\qpoint_0, \ldots, \qpoint_{n-k}}(t)
\end{align*}
where $\qpoint_i = \sum_{j = 0}^{k} \binom{k}{j} (-1)^{k-j}\bpoint_{i+j}$ for $i = 0, \ldots, n-k$.

%

\end{document}

%% file: packages.tex
\usepackage{xcolor}
\usepackage{comment}

\usepackage{cite}

\usepackage[mathcal]{euscript}

\usepackage[cmex10]{amsmath}
\usepackage{amssymb}

\usepackage{amsthm} 
  
\usepackage{dsfont} 

\usepackage{epsfig} 
\usepackage{tikz} 
\usetikzlibrary{positioning}
\usetikzlibrary{backgrounds}
\usetikzlibrary{calc}
\usepackage{multirow}
\usepackage{enumerate}

\usepackage[ruled, vlined, linesnumbered]{algorithm2e}

%% file: commands.tex
\newtheoremstyle{plain}
	  {}
	  {}
	  {\itshape}
	  {}
	  {\bfseries}
	  {}
	  {5pt plus 1pt minus 1pt}
	  {}

\newtheoremstyle{definition}
  	  {}
	  {}
	  {\normalfont}
	  {}
	  {\bfseries}
	  {}
	  {5pt plus 1pt minus 1pt}
	  {}
  
\theoremstyle{plain}

\newtheorem{lemma}{Lemma}
\newtheorem{proposition}{Proposition}

\newtheorem{property}{Property}
	
\theoremstyle{definition}
\newtheorem{definition}{Definition}

\newcommand{\refeq}[1]			{(\ref{#1})} 
\newcommand{\reffig}[1]			{Fig. \ref{#1}} 
\newcommand{\refsec}[1]			{Section \ref{#1}}
\newcommand{\refapp}[1]			{Appendix \ref{#1}}
\newcommand{\reftab}[1]			{Table \ref{#1}}

\newcommand{\reflem}[1]			{Lemma \ref{#1}}
\newcommand{\refdef}[1]			{Definition \ref{#1}}

\newcommand{\reffn}[1] 		    {\textsuperscript{\ref{#1}}}
\newcommand{\refpropty}[1] 	    {Property \ref{#1}}

%% file: notation.tex
\newcommand{\cdim}		{d} 
\newcommand{\cdegree}	{n} 

\newcommand{\bcurve}	{\mathrm{B}} 
\newcommand{\bpoly}		{b} 
\newcommand{\bdegree}	{\cdegree} 
\newcommand{\bdim} 		{\cdim} 
\newcommand{\bpoint}	{\vect{p}} 
\newcommand{\bpmat}		{\mat{P}} 
\newcommand{\bbasis}	{\vect{b}} 
\newcommand{\bvar}{\sigma^2} 
\newcommand{\bmean}{\mu} 
\newcommand{\bpmean}{\overline{\mu}} 
\newcommand{\bconsdist}{\mathcal{C}}

\newcommand{\ppoint} 	{\vect{p}} 
\newcommand{\qpoint}		{\vect{q}} 
\newcommand{\pmat} 	{\mat{P}} 
\newcommand{\qmat}		{\mat{Q}} 
\newcommand{\Dmat}    {\mat{D}} 
\newcommand{\Hmat}    {\mat{H}} 
\newcommand{\HmatN}   {\mat{H}_{\mathrm{N}}} 

\newcommand{\HmatDN}{\mat{H}_{\mathrm{DN}}} 
\newcommand{\HmatDDN}{\mat{H}_{\overline{\mathrm{DN}}}} 
\newcommand{\HmatDV}{\mat{H}_{\mathrm{DV}}} 
\newcommand{\HmatDDV}{\mat{H}_{\overline{\mathrm{DV}}}} 

\newcommand{\HmatB}{\mat{H}_{\mathrm{B}}} 
\newcommand{\Smat}  {\mat{S}} 

\newcommand{\freespace}{\mathcal{F}}
\newcommand{\obstspace}{\mathcal{O}}
\newcommand{\safecor} {\mathcal{SC}}
\newcommand{\cpoint}  {\vect{c}}
\newcommand{\xpoint}{\vect{x}}
\newcommand{\Xmat} {\mat{X}}

\newcommand{\refpath}{\mathrm{r}}

\newcommand{\R}  	{\mathbb{R}} 

\newcommand{\conv}      {\mathrm{conv}} 

\newcommand{\N}  	{\mathbb{N}} 
\newcommand{\PSDM} 	{S_{+}} 

\let\originalleft\left
\let\originalright\right
\renewcommand{\left}{\mathopen{}\mathclose\bgroup\originalleft}
\renewcommand{\right}{\aftergroup\egroup\originalright}
	
\newcommand{\plist}[1] 	{\left(#1\right)} 
\newcommand{\blist}[1]	{\left[ #1 \right]} 
\newcommand{\clist}[1]	{\left\{#1\right\}} 

\newcommand{\vect}[1]   {\mathrm{#1}}
\newcommand{\mat}[1]    {\mathbf{#1}}
\newcommand{\tr}[1] {{#1}^{\mathrm{T}}} 
\newcommand{\norm}[1]  {\|#1\|} %
\newcommand{\trace} {\mathrm{tr}} 

\newcommand{\ldf}   {:=} 

\newcommand{\argmin}{\operatornamewithlimits{arg\ min}} 
\newcommand{\argmax}{\operatornamewithlimits{arg\ max}} 
\newcommand{\diff} {\mathrm{d}}

\makeatletter
\DeclareRobustCommand\widecheck[1]{{\mathpalette\@widecheck{#1}}}
\def\@widecheck#1#2{%
    \setbox\z@\hbox{\m@th$#1#2$}%
    \setbox\tw@\hbox{\m@th$#1%
       \widehat{%
          \vrule\@width\z@\@height\ht\z@
          \vrule\@height\z@\@width\wd\z@}$}%
    \dp\tw@-\ht\z@
    \@tempdima\ht\z@ \advance\@tempdima2\ht\tw@ \divide\@tempdima\thr@@
    \setbox\tw@\hbox{%
       \raise\@tempdima\hbox{\scalebox{1}[-1]{\lower\@tempdima\box
\tw@}}}%
    {\ooalign{\box\tw@ \cr \box\z@}}}
\makeatother

%% file: drawings/tikz_first_order_discrete_difference.tex
%
%
%

%
%
\begin{tikzpicture}[scale=1.3, thick, font=\large, roundnode/.style={circle, draw}]
    \node[roundnode, label={[font=\large]below:$\bpoint_0$}] (0)at (0,0) {}; 
    \node[roundnode, label={[font=\large]below:$\bpoint_1$}] (1)at (1,0) {}; 
    \node[roundnode, label={[font=\large]below:$\bpoint_{n-1}$}] (2)at (2,0) {};
    \node[roundnode, label={[font=\large]below:$\bpoint_n$}] (3)at (3,0) {}; 
     
    \draw (0) to node[above]{1} (1); 
    \draw [dashed] (1) to node[above]{1} (2); 
    \draw (2) to node[above]{1} (3); 
\end{tikzpicture} 

%% file: drawings/tikz_first_order_difference_variance.tex
%
%
%

%
%
\begin{tikzpicture}[scale=1.3, thick, font=\large, roundnode/.style={circle, draw}]
    \node[roundnode, label={[font=\large]below:$\bpoint_0$}] (0)at (0,0) {}; 
    \node[roundnode, label={[font=\large]below:$\bpoint_1$}] (1)at (1,0) {}; 
    \node[roundnode, label={[font=\large]below:$\bpoint_{n-1}$}] (2)at (2,0) {};
    \node[roundnode, label={[font=\large]below:$\bpoint_n$}] (3)at (3,0) {}; 
     
    \draw (0) to node[above]{$n$} (1); 
    \draw [dashed] (1) to node[above]{$n$} (2); 
    \draw (2) to node[above]{$n$} (3); 
    \draw (0) to[out=60,in=120,looseness=0.6] node[above]{-1} (3);

\end{tikzpicture} 

%% file: drawings/tikz_second_order_discrete_difference_normAlt1.tex
\begin{tikzpicture}[scale=1.3, thick, font=\large, roundnode/.style={circle, draw}]
    \node[roundnode, label={[font=\large]below:$\bpoint_0$}] (0) at (0,0) {}; 
    \node[roundnode, label={[font=\large]below:$\bpoint_1$}] (1) at (1,0) {};
    \node[roundnode, label={[font=\large]below:$\bpoint_2$}] (2) at (2,0) {};
    \node[roundnode, label={[font=\large]below:$\bpoint_{n-2}$}] (3) at (3,0) {}; 
    \node[roundnode, label={[font=\large]below,xshift=-2pt:$\bpoint_{n-1}$}] (4) at (4,0) {};
    \node[roundnode, label={[font=\large]below:$\bpoint_n$}] (5) at (5,0) {}; 
     
    \draw (0) to node[above]{$2$} (1); 
    \draw (1) to node[above]{$4$} (2); 
    \draw [dashed] (2) to node[above]{$4$} (3);
    \draw (3) to node[above]{$4$} (4); 
    \draw (4) to node[above]{$2$} (5); 
    \draw (0) to[out=-70,in=-110,looseness=0.7] node[below]{-1} (2);
    \draw [dashed](1) to[out=70,in=110,looseness=0.7] node[above]{-1} (3);
    \draw [dash pattern=on 3pt off 3pt on 3pt off 5pt] (2) to[out=-70,in=-110,looseness=0.7] node[below]{-1} (4);
    \draw (3) to[out=70,in=110,looseness=0.7] node[above]{-1} (5);
\end{tikzpicture} 

%% file: drawings/tikz_first_order_continuous_derivative_n2.tex
\begin{tikzpicture}[scale=1.3, thick, font=\large, roundnode/.style={circle, draw}]

\node[roundnode] (0) at (-0.8660, -0.5000) {$0$}; 
\node[roundnode] (1) at (0.0000, 1.0000) {$1$}; 
\node[roundnode] (2) at (0.8660, -0.5000) {$2$}; 
    
\draw (0) to node[midway, left, xshift=-3pt]{$1$} (1); 
\draw (1) to node[midway, right, xshift=3pt]{$1$} (2); 
\draw (0) to node[midway, below, xshift=0pt]{$1$} (2); 

\end{tikzpicture}

%% file: drawings/tikz_first_order_continuous_derivative_n3.tex
\begin{tikzpicture}[scale=2.0, thick, font=\large, roundnode/.style={circle, draw}]

\node[roundnode] (0) at (0,0) {$0$}; 
\node[roundnode] (1) at (0,1) {$1$}; 
\node[roundnode] (2) at (1,1) {$2$}; 
\node[roundnode] (3) at (1,0) {$3$}; 

\draw (0) to node[midway, left, yshift=0pt]{$3$} (1);
\draw (0) to node[midway, left, xshift=-4pt]{$2$} (2);
\draw (0) to node[midway, below, yshift=0pt]{$1$} (3);
\draw (1) to node[midway, above, yshift=0pt]{$-1$} (2);
\draw (1) to node[midway, right, xshift=4pt]{$2$} (3); 
\draw (2) to node[midway, right, yshift=0pt]{$3$} (3); 
\end{tikzpicture}

%% file: drawings/tikz_first_order_continuous_derivative_n4.tex
\begin{tikzpicture}[scale=1.5, thick, font=\large, roundnode/.style={circle, draw}]

\node[roundnode] (0) at (-0.5878,-0.8090) {$0$}; 
\node[roundnode] (1) at (-0.9511,0.3090) {$1$}; 
\node[roundnode] (2) at (0.0000,1.0000) {$2$}; 
\node[roundnode] (3) at (0.9511,0.3090) {$3$}; 
\node[roundnode] (4) at (0.5878,-0.8090) {$4$}; 

\draw (0) to node[midway, left, xshift=-2pt]{$10$} (1); 
\draw (0) to node[left, xshift=0pt, yshift=+2pt]{$6$} (2); 
\draw (0) to node[below right, xshift=-2pt, yshift=+2pt]{$3$} (3);
\draw (0) to node[below, xshift=0pt]{$1$} (4);
\draw (1) to node[above left, xshift=0pt]{$-3$} (2); 
\draw (1) to node[above, xshift=0pt]{$2$} (3); 
\draw (1) to node[below left, xshift=+2pt, yshift=+2pt]{$3$} (4);
\draw (2) to node[above right, xshift=0pt]{$-3$} (3); 
\draw (2) to node[right, xshift=0pt, yshift=+2pt]{$6$} (4); 
\draw (3) to node[right, xshift=2pt]{$10$} (4); 

\end{tikzpicture}

%% file: drawings/tikz_first_order_discrete_difference_n2.tex
\begin{tikzpicture}[scale=1.3, thick, font=\large, roundnode/.style={circle, draw}]

\node[roundnode] (0) at (-0.8660, -0.5000) {$0$}; 
\node[roundnode] (1) at (0.0000, 1.0000) {$1$}; 
\node[roundnode] (2) at (0.8660, -0.5000) {$2$}; 
    
\draw (0) to node[midway, left, xshift=-3pt]{$1$} (1); 
\draw (1) to node[midway, right, xshift=3pt]{$1$} (2); 

\end{tikzpicture}

%% file: drawings/tikz_first_order_discrete_difference_n3.tex
\begin{tikzpicture}[scale=2.0, thick, font=\large, roundnode/.style={circle, draw}]

\node[roundnode] (0) at (0,0) {$0$}; 
\node[roundnode] (1) at (0,1) {$1$}; 
\node[roundnode] (2) at (1,1) {$2$}; 
\node[roundnode] (3) at (1,0) {$3$}; 

\draw (0) to node[midway, left, yshift=0pt]{$1$} (1); 
\draw (1) to node[midway, above, yshift=0pt]{$1$} (2); 
\draw (2) to node[midway, right, yshift=0pt]{$1$} (3); 

\end{tikzpicture}

%% file: drawings/tikz_first_order_discrete_difference_n4.tex
\begin{tikzpicture}[scale=1.5, thick, font=\large, roundnode/.style={circle, draw}]

\node[roundnode] (0) at (-0.5878,-0.8090) {$0$}; 
\node[roundnode] (1) at (-0.9511,0.3090) {$1$}; 
\node[roundnode] (2) at (0.0000,1.0000) {$2$}; 
\node[roundnode] (3) at (0.9511,0.3090) {$3$}; 
\node[roundnode] (4) at (0.5878,-0.8090) {$4$}; 

\draw (0) to node[midway, left, xshift=-2pt]{$1$} (1); 
\draw (1) to node[midway, above left, xshift=0pt]{$1$} (2); 
\draw (2) to node[midway, above right, xshift=0pt]{$1$} (3); 
\draw (3) to node[midway, right, xshift=2pt]{$1$} (4); 

\end{tikzpicture}

%% file: drawings/tikz_control_point_variance_n2.tex
\begin{tikzpicture}[scale=1.3, thick, font=\large, roundnode/.style={circle, draw}]

\node[roundnode] (0) at (-0.8660, -0.5000) {$0$}; 
\node[roundnode] (1) at (0.0000, 1.0000) {$1$}; 
\node[roundnode] (2) at (0.8660, -0.5000) {$2$}; 
    
\draw (0) to node[midway, left, xshift=-3pt]{$1$} (1); 
\draw (1) to node[midway, right, xshift=3pt]{$1$} (2); 
\draw (0) to node[midway, below, xshift=0pt]{$1$} (2); 

\end{tikzpicture}

%% file: drawings/tikz_control_point_variance_n3.tex
\begin{tikzpicture}[scale=2.0, thick, font=\large, roundnode/.style={circle, draw}]

\node[roundnode] (0) at (0,0) {$0$}; 
\node[roundnode] (1) at (0,1) {$1$}; 
\node[roundnode] (2) at (1,1) {$2$}; 
\node[roundnode] (3) at (1,0) {$3$}; 

\draw (0) to node[midway, left, yshift=0pt]{$1$} (1);
\draw (0) to node[midway, left, xshift=-4pt]{$1$} (2);
\draw (0) to node[midway, below, yshift=0pt]{$1$} (3);
\draw (1) to node[midway, above, yshift=0pt]{$1$} (2);
\draw (1) to node[midway, right, xshift=4pt]{$1$} (3); 
\draw (2) to node[midway, right, yshift=0pt]{$1$} (3); 
\end{tikzpicture}

%% file: drawings/tikz_control_point_variance_n4.tex
\begin{tikzpicture}[scale=1.5, thick, font=\large, roundnode/.style={circle, draw}]

\node[roundnode] (0) at (-0.5878,-0.8090) {$0$}; 
\node[roundnode] (1) at (-0.9511,0.3090) {$1$}; 
\node[roundnode] (2) at (0.0000,1.0000) {$2$}; 
\node[roundnode] (3) at (0.9511,0.3090) {$3$}; 
\node[roundnode] (4) at (0.5878,-0.8090) {$4$}; 

\draw (0) to node[midway, left, xshift=-2pt]{$1$} (1); 
\draw (0) to node[left, xshift=0pt, yshift=+2pt]{$1$} (2); 
\draw (0) to node[below right, xshift=-2pt, yshift=+2pt]{$1$} (3);
\draw (0) to node[below, xshift=0pt]{$1$} (4);
\draw (1) to node[above left, xshift=0pt]{$1$} (2); 
\draw (1) to node[above, xshift=0pt]{$1$} (3); 
\draw (1) to node[below left, xshift=+2pt, yshift=+2pt]{$1$} (4);
\draw (2) to node[above right, xshift=0pt]{$1$} (3); 
\draw (2) to node[right, xshift=0pt, yshift=+2pt]{$1$} (4); 
\draw (3) to node[right, xshift=2pt]{$1$} (4); 

\end{tikzpicture}

%% file: drawings/tikz_second_order_continuous_derivative_n2.tex
\begin{tikzpicture}[scale=1.3, thick, font=\large, roundnode/.style={circle, draw}]

\node[roundnode] (0) at (-0.8660, -0.5000) {$0$}; 
\node[roundnode] (1) at (0.0000, 1.0000) {$1$}; 
\node[roundnode] (2) at (0.8660, -0.5000) {$2$}; 
    
\draw (0) to node[above left, xshift=0pt]{$2$} (1); 
\draw (0) to node[below, xshift=0pt]{$-1$} (2); 
\draw (1) to node[above right, xshift=0pt]{$2$} (2); 

\end{tikzpicture}

%% file: drawings/tikz_second_order_continuous_derivative_n3.tex
\begin{tikzpicture}[scale=2.0, thick, font=\large, roundnode/.style={circle, draw}]

\node[roundnode] (0) at (0,0) {$0$}; 
\node[roundnode] (1) at (0,1) {$1$}; 
\node[roundnode] (2) at (1,1) {$2$}; 
\node[roundnode] (3) at (1,0) {$3$}; 

\draw (0) to node[midway, left, yshift=0pt]{$3$} (1);
\draw (0) to node[midway, below, yshift=0pt]{$-1$} (3);
\draw (1) to node[midway, above, yshift=0pt]{$3$} (2);
\draw (2) to node[midway, right, yshift=0pt]{$3$} (3); 

\end{tikzpicture}

%% file: drawings/tikz_second_order_continuous_derivative_n4.tex
\begin{tikzpicture}[scale=1.5, thick, font=\large, roundnode/.style={circle, draw}]

\node[roundnode] (0) at (-0.5878,-0.8090) {$0$}; 
\node[roundnode] (1) at (-0.9511,0.3090) {$1$}; 
\node[roundnode] (2) at (0.0000,1.0000) {$2$}; 
\node[roundnode] (3) at (0.9511,0.3090) {$3$}; 
\node[roundnode] (4) at (0.5878,-0.8090) {$4$}; 

\draw (0) to node[midway, left, xshift=-2pt]{$9$} (1); 
\draw (0) to node[left, xshift=+2pt, yshift=+2pt]{$-1$} (2); 
\draw (0) to node[below, xshift=+2pt, yshift=+2pt]{$-1$} (3);
\draw (0) to node[below, xshift=0pt]{$-1$} (4);
\draw (1) to node[above left, xshift=0pt]{$4$} (2); 
\draw (1) to node[above, xshift=0pt]{$4$} (3); 
\draw (1) to node[below, xshift=-4pt, yshift=+2pt]{$-1$} (4);
\draw (2) to node[above right, xshift=0pt]{$4$} (3); 
\draw (2) to node[right, xshift=-2pt, yshift=+2pt]{$-1$} (4); 
\draw (3) to node[right, xshift=2pt]{$9$} (4); 

\end{tikzpicture}

%% file: drawings/tikz_second_order_discrete_difference_n2.tex
\begin{tikzpicture}[scale=1.3, thick, font=\large, roundnode/.style={circle, draw}]

\node[roundnode] (0) at (-0.8660, -0.5000) {$0$}; 
\node[roundnode] (1) at (0.0000, 1.0000) {$1$}; 
\node[roundnode] (2) at (0.8660, -0.5000) {$2$}; 
    
\draw (0) to node[midway, left, xshift=-3pt]{$2$} (1); 
\draw (1) to node[midway, right, xshift=3pt]{$2$} (2); 
\draw (0) to node[midway, below, xshift=0pt]{$-1$} (2); 

\end{tikzpicture}

%% file: drawings/tikz_second_order_discrete_difference_n3.tex
\begin{tikzpicture}[scale=2.0, thick, font=\large, roundnode/.style={circle, draw}]

\node[roundnode] (0) at (0,0) {$0$}; 
\node[roundnode] (1) at (0,1) {$1$}; 
\node[roundnode] (2) at (1,1) {$2$}; 
\node[roundnode] (3) at (1,0) {$3$}; 

\draw (0) to node[midway, left, yshift=0pt]{$2$} (1); 
\draw (1) to node[midway, above, yshift=0pt]{$4$} (2); 
\draw (2) to node[midway, right, yshift=0pt]{$2$} (3); 
\draw (0) to node[midway, left, xshift=-3pt]{$-1$} (2); 
\draw (1) to node[midway, right, xshift=3pt]{$-1$} (3); 

\end{tikzpicture}

%% file: drawings/tikz_second_order_discrete_difference_n4.tex
\begin{tikzpicture}[scale=1.5, thick, font=\large, roundnode/.style={circle, draw}]

\node[roundnode] (0) at (-0.5878,-0.8090) {$0$}; 
\node[roundnode] (1) at (-0.9511,0.3090) {$1$}; 
\node[roundnode] (2) at (0.0000,1.0000) {$2$}; 
\node[roundnode] (3) at (0.9511,0.3090) {$3$}; 
\node[roundnode] (4) at (0.5878,-0.8090) {$4$}; 

\draw (0) to node[midway, left, xshift=-2pt]{$2$} (1); 
\draw (1) to node[midway, above left, xshift=0pt]{$4$} (2); 
\draw (2) to node[midway, above right, xshift=0pt]{$4$} (3); 
\draw (3) to node[midway, right, xshift=2pt]{$2$} (4); 
\draw (0) to node[below left, xshift=0pt]{$-1$} (2); 
\draw (1) to node[below, xshift=0pt]{$-1$} (3); 
\draw (2) to node[midway,below right, xshift=0pt]{$-1$} (4); 

\end{tikzpicture}

%% file: drawings/tikz_control_point_jensen_gap_n2.tex
\begin{tikzpicture}[scale=1.3, thick, font=\large, roundnode/.style={circle, draw}]

\node[roundnode] (0) at (-0.8660, -0.5000) {$0$}; 
\node[roundnode] (1) at (0.0000, 1.0000) {$1$}; 
\node[roundnode] (2) at (0.8660, -0.5000) {$2$}; 
    
\draw (0) to node[above left, xshift=0pt]{$2$} (1); 
\draw (0) to node[below, xshift=0pt]{$-1$} (2); 
\draw (1) to node[above right, xshift=0pt]{$2$} (2); 

\end{tikzpicture}

%% file: drawings/tikz_control_point_jensen_gap_n3.tex
\begin{tikzpicture}[scale=2.0, thick, font=\large, roundnode/.style={circle, draw}]

\node[roundnode] (0) at (0,0) {$0$}; 
\node[roundnode] (1) at (0,1) {$1$}; 
\node[roundnode] (2) at (1,1) {$2$}; 
\node[roundnode] (3) at (1,0) {$3$}; 

\draw (0) to node[midway, left, yshift=0pt]{$3$} (1);
\draw (0) to node[midway, below, yshift=0pt]{$-1$} (3);
\draw (1) to node[midway, above, yshift=0pt]{$3$} (2);
\draw (2) to node[midway, right, yshift=0pt]{$3$} (3); 

\end{tikzpicture}

%% file: drawings/tikz_control_point_jensen_gap_n4.tex
\begin{tikzpicture}[scale=1.5, thick, font=\large, roundnode/.style={circle, draw}]

\node[roundnode] (0) at (-0.5878,-0.8090) {$0$}; 
\node[roundnode] (1) at (-0.9511,0.3090) {$1$}; 
\node[roundnode] (2) at (0.0000,1.0000) {$2$}; 
\node[roundnode] (3) at (0.9511,0.3090) {$3$}; 
\node[roundnode] (4) at (0.5878,-0.8090) {$4$}; 

\draw (0) to node[midway, left, xshift=-2pt]{$4$} (1); 
\draw (1) to node[midway, above left, xshift=0pt]{$4$} (2); 
\draw (2) to node[midway, above right, xshift=0pt]{$4$} (3); 
\draw (3) to node[midway, right, xshift=2pt]{$4$} (4); 
\draw (0) to node[midway, below, xshift=0pt]{$-1$} (4);

\end{tikzpicture}